\documentclass[letterpaper,twocolumn,10pt]{article}
\usepackage{usenix,epsfig,endnotes}
\RequirePackage{setspace}
\usepackage[dvipsnames,table,xcdraw]{xcolor}
\usepackage{subfig}
\usepackage{amssymb,amsmath}
\usepackage{pgfplots}
\usepackage{pdfpages}
\usepackage{listings}
\usepackage{parcolumns}
\usepackage{sectsty}

\usepackage{amssymb,amsmath}
\usepackage{amsthm}
\usepackage{textcomp}
\usepackage{float}
\usepackage{graphicx}
\usepackage[font={small,bf},labelfont={small,bf}]{caption}

\newtheorem{lemma}{Lemma}
\newcommand{\assign}{:=}

\newtheorem{theorem}{Theorem}

\newcommand{\Tool}{{GaDei}\xspace}

\newcommand\jewel{{ Joule }}
\newcommand\welltok{{ Watt }}
\usepackage[normalem]{ulem}
\usepackage{xspace}
\usepackage{titlesec}
\usepackage{smartdiagram}
\usepackage{url}
\usesmartdiagramlibrary{additions}

\titlespacing*{\section} {0pt}{1.2ex}{1.0ex}
\titlespacing*{\subsection} {0pt}{0.8ex}{0.7ex}
\titlespacing*{\subsubsection}{0pt}{0.4ex}{0.3ex}

\begin{document}

\setlength{\belowdisplayskip}{0pt} \setlength{\belowdisplayshortskip}{0pt}
\setlength{\abovedisplayskip}{0pt} \setlength{\abovedisplayshortskip}{0pt}

\title{
\Tool: On Scale-up Training As A Service For Deep Learning}
\author{
{\rm Wei Zhang$^{1}$, Minwei Feng$^{2}$, Yunhui Zheng$^{1}$, Yufei Ren$^{1}$, Yandong Wang$^{1}$}\\
{\rm Ji Liu$^{3}$, Peng Liu$^{1}$, Bing Xiang$^{2}$, Li Zhang$^{1}$, Bowen Zhou$^{2}$, Fei Wang$^{4}$}\\
IBM T.J.Watson Research$^{1}$, IBM Watson$^{2}$, University of Rochester$^{3}$, Cornell University$^{4}$
}
\maketitle

\thispagestyle{empty}

\begin{abstract}
Deep learning (DL) training-as-a-service (TaaS) is an important emerging industrial workload. TaaS must satisfy a wide range of customers who have no experience and/or resources to tune DL hyper-parameters (e.g., mini-batch size and learning rate), and meticulous tuning for each user's dataset is prohibitively expensive. Therefore, TaaS hyper-parameters must be fixed with values that are applicable to all users. Unfortunately, few research papers have studied how to design a system for TaaS workloads. By evaluating the IBM Watson Natural Language Classfier (NLC) workloads,  the most popular IBM cognitive service used by thousands of enterprise-level clients globally, we provide empirical evidence that only the conservative hyper-parameter setup (e.g., small mini-batch size) can guarantee acceptable model accuracy for a wide range of customers.  
Unfortunately, smaller mini-batch size requires higher communication bandwidth in a parameter-server based DL training system. In this paper, we characterize the exceedingly high communication bandwidth requirement of TaaS using representative industrial deep learning workloads. 
We then present \Tool, a highly optimized shared-memory based scale-up parameter server design. 
We evaluate \Tool using both commercial benchmarks and public benchmarks and demonstrate that \Tool significantly outperforms the state-of-the-art parameter-server based implementation while maintaining the required accuracy. \Tool achieves near-best-possible runtime performance, constrained only by the hardware limitation. Furthermore, to the best of our knowledge, \Tool is the only scale-up DL system that provides fault-tolerance. 

\end{abstract}
\section{Introduction}
\label{sec:intro}


When deployed on the cloud, deep learning (DL) training-as-a-service (TaaS) faces unique challenges: different customers upload their own training data and expect a model of high prediction accuracy returned shortly. Unlike academic researchers, customers have neither expertise nor resources to conduct time-consuming hyper-parameter (e.g., learning rate, mini-batch size) tuning. Hyper-parameter tuning itself is an unsolved and challenging research topic and the tuning process is usually prohibitively expensive \cite{hyper_param_search, hyper_param_bayes}.
The goal of TaaS is not to provide each user a dedicated hyper-parameter setup and the companioning model, but to provide all users an unified DL model and the common hyper-parameter setup which still delivers the cutting-edge model to customers.
  As a result, \textbf{industrial practitioners adopt conservative hyper-parameter setup} (e.g., small mini-batch size and small learning rate). On the other hand, a training system that can support such conservative setup can easily support less restrictive setup, which makes hyper-parameter tuning turn-around time much shorter. 
 
In a parameter server (PS) based DL system, such a conservative setup implies high-frequency communication with the PS. We provide a detailed analysis of the communication 
bandwidth requirement for real-world commercial workloads in Section~\ref{sec:nlc} and Section~\ref{sec:modeling}. The analysis shows that the bandwidth requirement is beyond the capacity of advanced communication hardware (e.g., RDMA). Furthermore, with faster GPU devices and more efficient software library support, such as cuDNN~\cite{cudnn}, the communication cost of exchanging models between parameter servers and learners start to dominate the training cost, which renders scale-out training unattractive. 

In addition, the staleness 
issue inherent in distributed 
deep learning systems makes scaling-out deep learning less cost-effective when 
training over a dozen GPU learners~\cite{liu-asgd-nips-2015,zhang-ijcai-2016}. As a result, one of the largest public-known commercial deep 
learning training system~\cite{deep-speech}  uses 8 GPUs. \textbf{Scale-up} deep learning training becomes a favored approach in the TaaS scenario. 


Although powerful scale-up servers, such as NVIDIA DevBox, provides hardware platforms to improve the training 
performance, our evaluation (detailed in Section~\ref{sec:eval-compare}) reveals that 
the state-of-the-art scale-up software solutions  are  unable to make the best use 
of underlying hardware. Two representative open-source scale-up deep learning 
frameworks are mpiT\cite{mpiT}, an MPI-based ASGD\footnote{Asynchronous Stochastic Gradient Descent (ASGD) is defined in 
Section~\ref{sec:defs}} framework; and DataParallelTable (DPT), a 
nccl-based (nccl\cite{nccl} is a high-performance in-node GPU collectives implementation) 
SSGD\footnote{Synchronous Stochastic Gradient Descent (SSGD) is defined in Section~\ref{sec:defs}} framework. The former 
solution incurs unnecessary memory copies between GPUs and MPI runtime and 
is unable to implement the lock-free update as proposed in HogWild!\cite{hogwild}. 
The latter solution incurs synchronization barriers by forcing all GPUs to operate in lock-steps, which leads to the straggler problem. 
Furthermore, none of these two solutions provide a fault-tolerance 
mechanism, which makes them undesirable for the commercial adoption.

To solve these issues, we introduce \Tool, a high-performance 
scale-up parameter server design, aiming to efficiently coordinate the model 
synchronization among GPUs located on the same machine. 
\Tool strives to pipeline the entire model 
synchronization, overlapping all the model training and data movement 
phases to eliminate GPU stalls. Specifically, \Tool implements three 
system optimizations: (1) Communication via minimum memory copies (2) Lock-free Hogwild! style weights update rule 
(3) On-device double buffering, along with GPU multi-streaming, 
to pipeline model trainings and parameter movements. \Tool enables 
training with small mini-batch size, which mitigates the staleness 
issue to guarantee model 
convergence.  By evaluating \Tool on a diverse set of real-world deep 
learning workloads, we demonstrate that \Tool is able to efficiently 
exploit the bandwidth offered by the commodity scale-up servers, 
providing faster convergence with significantly higher training 
speedups compared to existing open-source solutions, such as mpiT and DPT.

Overall, this work has made the following contributions:
\begin{enumerate}
\item We have identified the key challenges in designing a training-as-a-service system: Hyper-parameters must be set conservatively (e.g.,  small mini-batch size and high model communication frequency) to guarantee model accuracy. 

\item We have designed and implemented \Tool, a highly-optimized parameter 
server system, to deliver scale-up and resilient 
training for TaaS workloads on multi-GPU servers. \Tool enables efficient multi-learner training for arbitrary type of neural networks (e.g., CNN, RNN). The design principle of \Tool is independent from the underlying gradient-calculation building-blocks and can complement any open-source DL frameworks (e.g., Torch\cite{torch}, Caffe\cite{caffe}, and TensorFlow\cite{tensorflow}).
\item We have proved that \Tool's system design guarantees both model convergence and deadlock-free. To the best of our knowledge, \Tool is the only scale-up parameter server design that provides both fault-tolerance and deadlock-free guarantee. 
We have systematically evaluated \Tool's performance by 
using 6 deep learning workloads with 3 state-of-the-art deep learning models. Evaluation results demonstrate \Tool often outperforms state-of-the-art solutions by an order of magnitude.

\end{enumerate}



\section{Background and Motivation}
In this section, we introduce the background and define the terminologies (in bold font) used in this paper. 
Then we describe the characteristics of TaaS workloads. Finally, we theoretically justify why our design choice guarantees the acceptable model accuracy. 

\subsection{Terminology Definition}

\label{sec:defs}
In essence, deep learning  solves the following generic optimization problem 
\[ \min_\theta \quad F (\theta) \assign \frac{1}{N} \sum_{n = 1}^N f_n(\theta), \]\\ 
where $\theta$ is the \textbf{parameter} (or \textbf{weights}) vector we are seeking, $N$ is the number of samples, and $f_n(\theta)$ is the loss function for the $n^{\text{th}}$ sample. $f(\theta)$ is typically in a form of a multi-layered neural network. 
Stochastic Gradient Descent(SGD) is the de facto algorithm to solve the deep learning optimization problem.  SGD iterates over each training sample and applies Equation~\ref{eqn:gd} to update weights. In Equation~\ref{eqn:gd}, $i$ is the iteration number, $k$ is the k-th parameter, $\nabla$ is the differential operator, and $\alpha$ is the \textbf{learning rate}. Using a large learning rate may converge faster but it may also overshoot so that it does not converge at all, thus using a smaller learning rate is a safer choice in the production run.  SGD passes through the entire training dataset several times until the model converges. Each pass is called an \textbf{epoch} (denoted as $E$). 
To improve computation efficiency, one can group a number of 
samples (i.e., a \textbf{mini-batch}, the size of one mini-batch is denoted as $\mu$) and 
apply Equation~\ref{eqn:gd} to update weights $\theta^{(k)}(i)$ for the $i+1$-th mini-batch.
\begin{equation}
  \theta^{(k)}(i+1) = \theta^{(k)}(i) - \alpha\nabla \theta^{(k)}(i) \label{eqn:gd}
\end{equation}


\begin{figure}[!t]
\centering
\begin{minipage}{.56\linewidth}
  \centering
  \includegraphics[width=1\linewidth]{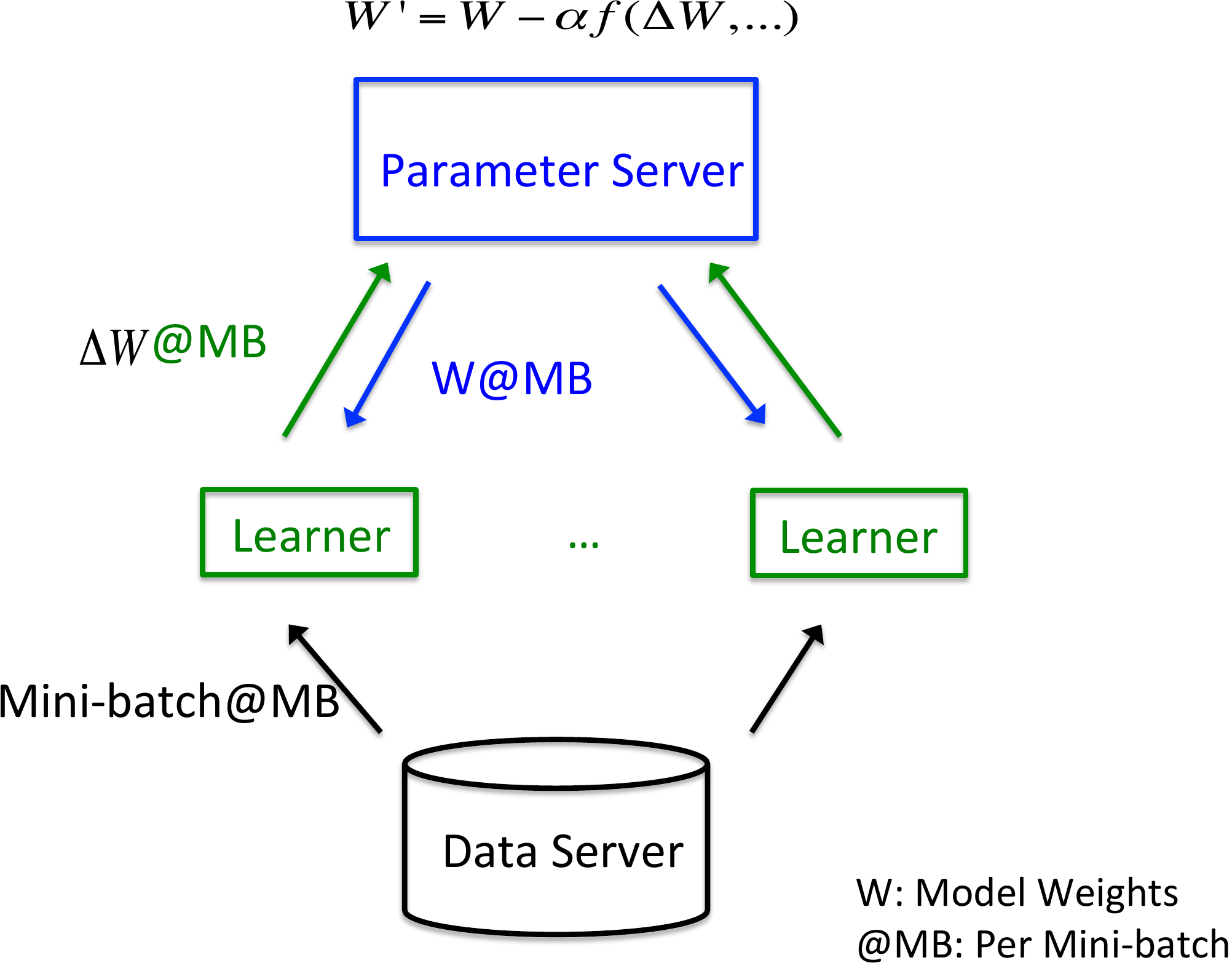}
  \captionof{figure}{\small A typical parameter server architecture.}
  \label{fig:design_base}
\end{minipage}%
\hspace{0.2cm}
\begin{minipage}{.4\linewidth}
  \centering
  \includegraphics[trim={0 8.1cm 0 0},clip,width=.8\linewidth]{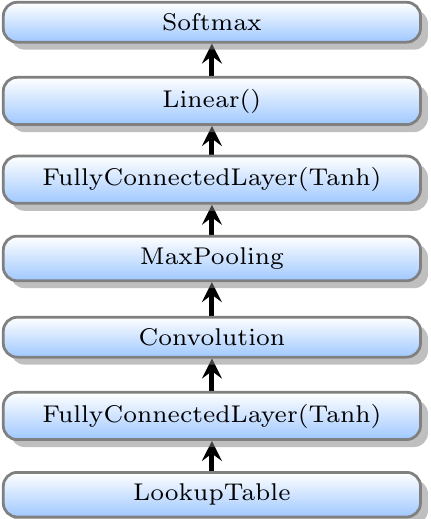}
  \captionof{figure}{\small Convolutional Neural Network based model.}
  \label{fig:nlc_model}
\end{minipage}
\end{figure}

To accelerate deep learning training, practitioners usually adopt the \textbf{Parameter Server} (PS) 
architecture as illustrated in Figure~\ref{fig:design_base}. Each \textbf{learner} retrieves a 
mini-batch of training samples from data storage, calculates gradients and sends the gradients 
to the PS. PS then updates its weights using received gradients. Before calculating 
the next gradients, each learner pulls weights from the PS. 
We use $\lambda$ to represent number of learners. Two most-widely adopted parameter server 
communication protocols are Synchronous SGD (\textbf{SSGD}) and Asynchronous SGD (\textbf{ASGD}). 
In SSGD, the PS 
collects gradients from each learner and then updates weights following the rule defined in 
Equation~\ref{eqn:ssgd}. SSGD is mathematically equivalent to SGD when $\lambda \times \mu$ is equal to the mini-batch size used in SGD.    
\begin{equation}
	\begin{aligned}
		\nabla \theta^{(k)}(i) &= \frac{1}{\lambda}{\sum_{l=1}^{\lambda} \nabla\theta_{l}^{(k)}}\\
		\theta^{(k)}(i+1) &= \theta^{(k)}(i) - \alpha\nabla \theta^{(k)}(i)
	\end{aligned}
 \label{eqn:ssgd}
\end{equation}

SSGD is not computationally efficient because the PS stalls the learners until it finishes collecting the gradients from all learners and updating the 
weights. ASGD relaxes such constraints by applying gradient update rule defined in Equation~\ref{eqn:asgd}. 
\begin{equation}
\begin{aligned}
\nabla \theta^{(k)}(i) &= \nabla\theta_{l}^{(k)}, L_{l} \in {L_{1}, ..., L_{\lambda}} \\
\theta^{(k)}(i+1) &= \theta^{(k)}(i) - \alpha\nabla \theta^{(k)}(i)
\end{aligned} 
\label{eqn:asgd}
\end{equation}

In ASGD, whenever PS receives a gradient from any learner, PS starts updating its weights. ASGD has the obvious runtime performance advantage over SSGD because learners do not wait for each other to start communicating with the PS. On the other hand, PS and the learners see different weights. PS always has the most up-to-date weights and the discrepancy between the weights used in a learner and the weights stored on PS  is measured by staleness. When PS updates the weights, it increments the weights's (scalar) timestamp by 1; \textbf{staleness} is defined as the difference between the timestamp of the learner's weights and the timestamp of PS's weights. 
A large staleness can cause learners to mis-calculate the gradients, which leads to a slower convergence rate\cite{bounded_staleness, distbelief, zhang-ijcai-2016, zhang-icdm-2016}.  We explain why training with a  small mini-batch size can effectively reduce staleness in Section~\ref{sec:smallbs_justification}. Additionally, several recent works~\cite{mpiT,feng:2016, liu-asgd-nips-2015, zhang-ijcai-2016} demonstrate that ASGD can converge to a similar model accuracy ($\pm 1\%$) as SGD after training with the same number of epochs, when the staleness is bounded in the system (typically up to a dozen of learners in the system). Assuming $T_{1}$ is the time for a single learner SGD algorithm to train $E$ epochs and $T_{2}$ is the time for ASGD algorithm to train $E$ epochs,  ASGD \textbf{speed up} is defined as $\frac{T1}{T2}$. For a fair comparison, one must also certify that the model accuracy trained by ASGD is similar to that of SGD after $E$ epochs.


\subsection{Characteristics of the TaaS Workloads}
\label{sec:nlc}

In this section, we detail the characteristics of the training-as-a-service workloads by studying IBM Watson's natural language classification (NLC) service, which is the most popular service on IBM Watson's cognitive computing cloud and used by thousands of enterprise-level customers globally. 


The NLC task is to classify input sentences into a target category in a predefined 
label set. NLC has been extensively used in practical applications, 
including sentiment analysis, topic classification, question classification, 
author profiling, intent classification, and even bug detection, etc. 
State-of-the-art method for NLC is based on deep 
learning~\cite{dp,Schmidhuber15,Goodfellow-et-al-2016-Book, 
kalchbrenner-grefenstette-blunsom:2014,kim:2014:EMNLP2014}. 
Figure \ref{fig:nlc_model} 
illustrates the 
deep learning model used in the NLC service.

NLC service is deployed as \textbf{''training-as-a-service''} in the cloud. 
After customers upload their in-house training data to the cloud, the NLC model 
training will be triggered in the background. The NLC model is ready to use after the training completes. 
Hence from the customer perspective, the turn around time is the model training time.
Although the deep learning brings superior classification accuracy, one known 
issue is the time-consuming training phase, which can severely affect customer's experience.  \textbf{Minimizing the training time has become IBM Watson NLC's top priority.}


In order to improve the performance for neural network training, previous 
work has attempted to use scale-out frameworks to coordinate learners distributed on different 
computing nodes. Good speedup on image recognition tasks like  
ImageNet\cite{imagenet} has been reported. However, this type of tasks does 
not represent the commercial TaaS workloads. By examining the real TaaS 
workloads, we find that corpus size of training sets are 
generally less than 10k for most use-cases and usually they come with a diverse set of labels. The reason is that in practice annotating training data is expensive for most customers. In addition, we have 
identified the following characteristics that are critical to the quality of trained model:

\noindent\textbf{(1) Large batch sizes can incur significant accuracy loss}: 
It is well-known that using large mini-batch can improve GPU utilization and incur less demand for communication; however, this runtime performance improvement 
must not sacrifice model accuracy. Our field study reveals that the deep learning method used in this paper is on average 3\%-6\% more accurate than other much less computation-intensive methods (e.g., SVM) for NLC tasks.  Therefore, an accuracy loss larger than 3\%-6\% will invalidate the use of deep learning for NLC tasks. To study the impact of different batch sizes on model accuracy, we use 4 representative NLC datasets (the detailed description of each task is given in Table~\ref{tab:task_description}) 
and evaluate the accuracy under different batch sizes in 
Figure~\ref{fig:nlc-batch-accuracy}. The batch size is increased from 1 to 128. Experimental results 
show that using large batch sizes would result in unacceptable accuracy loss for 
three out of four cases. When using large mini-batch size, models for challenging workloads (a large amount of labels and little training data for each label) such as \jewel and \welltok do not even converge. Other researchers have also observed that large batch size slows down convergence~\cite{smallbatch_wilson, smallbatch_li,smallbatch_keskar}. 

\begin{figure}[thb]    
    \vspace{-0.0pc} 
    \begin{center}    
    {\includegraphics[width=0.46\textwidth]{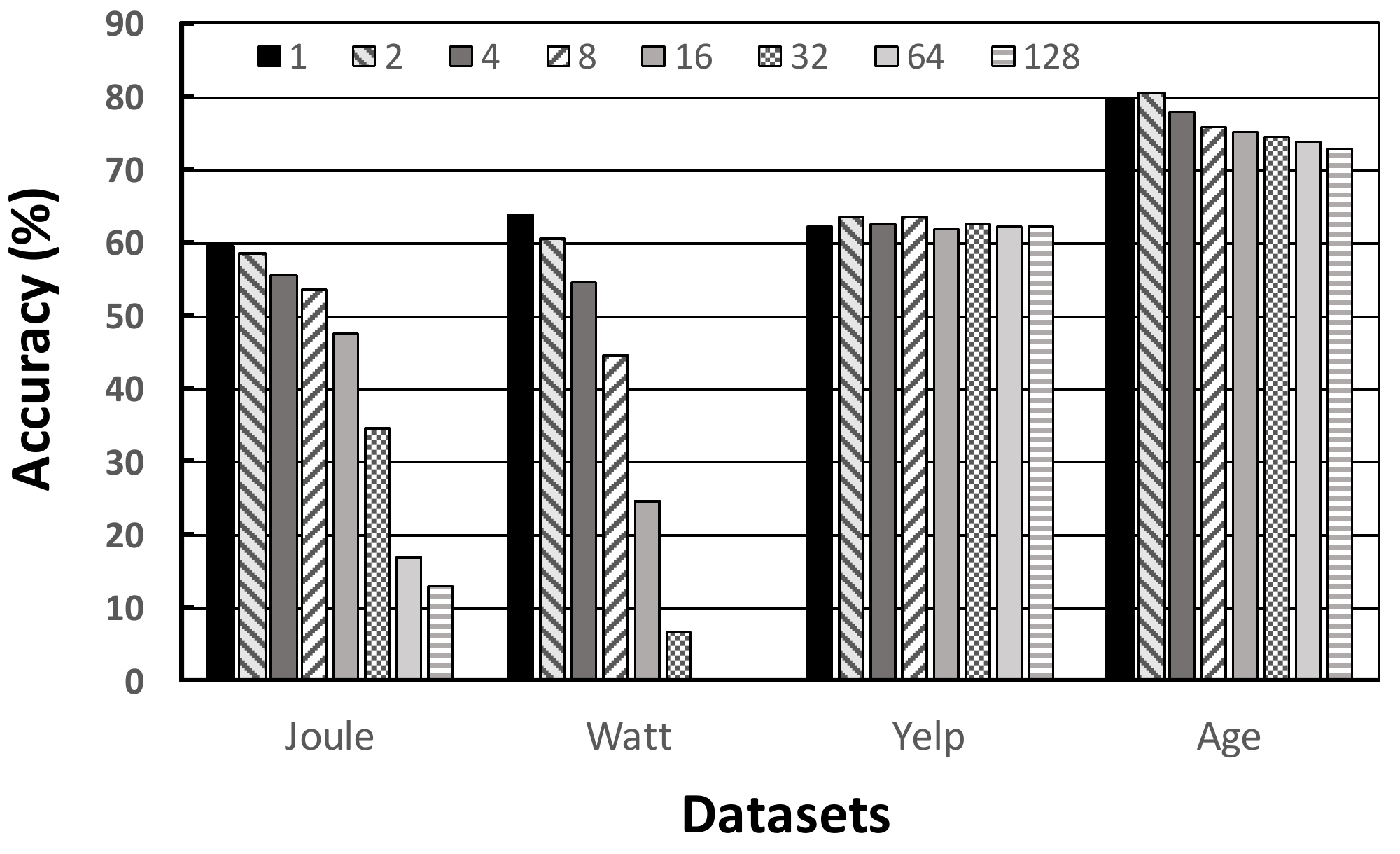}}
    \caption{\small Accuracy performance of different datasets under different batch sizes. Mini-batch size ranges from 1 to 128. The number of training epochs is fixed at 200. Using large mini-batch size severely decreases model accuracy. }
    \label{fig:nlc-batch-accuracy}
    \end{center}   
\end{figure}

\noindent\textbf{(2) Low communication frequency decreases model quality}: 
Communication bandwidth is typically much lower in the cloud than in the HPC systems. 
One way of mitigating the communication bottleneck is to  
allow each learner to process many mini-batches before it synchronizes 
with the parameter server. However, although less-frequent communication can efficiently 
increase the GPU utilization, but it can severely decrease model accuracy as shown in Table~\ref{tab:staleness-utilization}. Intuitively, less frequent communication causes higher discrepancy (i.e., staleness) between learners and the PS and it will decrease model accuracy. Therefore, learners should communicate with the PS as frequently as possible. Ideally, each learner shall communicate with the PS after each mini-batch training.
\begin{table}[t]
\scriptsize
\centering
\begin{tabular}{|l|r|r|r|r|}
\hline
\multicolumn{ 1}{|l|}{} & \multicolumn{ 2}{c|}{\textbf{Accuracy}} & \multicolumn{ 2}{c|}{\textbf{GPU utilization}} \\ 
\multicolumn{ 1}{|l|}{\textbf{Workload}} & \multicolumn{ 2}{c|}{\textbf{(\%)}} & \multicolumn{ 2}{c|}{\textbf{(\%)}} \\ \cline{ 2- 5}
\multicolumn{ 1}{|l|}{} & \multicolumn{1}{l|}{\textbf{$CI$=1}} & \multicolumn{1}{l|}{\textbf{$CI$=8}} & \multicolumn{1}{l|}{\textbf{$CI$=1}} & \multicolumn{1}{r|}{\textbf{$CI$=8}} \\ \hline
\textbf{\jewel} & 57.70 & 39.80 & 7.90 & 35.80 \\ \hline
\textbf{\welltok} & 57.80 & 55.10 & 4.90 & 29.40 \\ \hline
\textbf{Age} & 74.10 & 63.60 & 4.00 & 24.90 \\ \hline
\textbf{MR} & 77.20 & 50.00 & 5.50 & 33.70 \\ \hline
\textbf{Yelp} & 56.90 & 20.00 & 8.80 & 42.90 \\ \hline
\end{tabular}
\caption{\small Accuracy loss and GPU utilization with different communication interval in the parameter server based training framework. $CI$ is the communication interval, measured as the number of mini-batches each learner has processed before it communicates with the parameter server. High $CI$ can improve GPU utilization but severely affects model accuracy. This experiment is conducted using the mpiT package, on a 12 NVIDIA K20-GPU cluster with a 10Gb/s interconnect.}
\label{tab:staleness-utilization}
\end{table}

\noindent\textbf{(3) Conservative hyper-parameter configuration is imperative}: 
Previous research focuses on how to speedup training for one specific dataset and heavy hyper parameter tuning is required to achieve the best possible accuracy. For example, Table~\ref{tab:motivation} shows the typical hyper-parameter setups for training CIFAR and ImageNet with AlexNet\cite{alexnet} model. The setups vary greatly for different datasets.   
In contrast, the TaaS users have neither expertise nor resources for hyper-parameter tuning. 
Customers just upload the 
training data and then expect a well trained model to be ready in a short 
period of time. As a result, the hyper-parameters have to be preset to 
fixed values to cover diverse use cases. Table~\ref{tab:nlc_setup} describes the hyper-parameter used in NLC. For thousands of different datasets, NLC adopts a much simpler and more conservative setup. Note this is a significant difference from commonly evaluated workloads (e.g., CIFAR and ImageNet), where hyper-parameter tuning is specific to each dataset/model and usually is a result of a multi-person-years effort\cite{hyper_param_search, hyper_param_bayes}. In TaaS, \textbf{conservative 
configurations with small batch size, high communication frequency, and small 
learning rate are commonly adopted to satisfy a wide range of users.}


  
\begin{table}[htbp]
\scriptsize
\centering
\begin{tabular}{|l|l|l|l|l|}
\hline
Dataset & \multicolumn{ 2}{l|}{CIFAR} & \multicolumn{ 2}{l|}{ImageNet} \\ \hline
Mini-batch size & \multicolumn{ 2}{r|}{128} & \multicolumn{ 2}{r|}{256} \\ \hline
Number of epochs & \multicolumn{ 2}{r|}{100} & \multicolumn{ 2}{r|}{40} \\ \hline
Learning Rate(LR) & Epoch & LR & Epoch & LR \\ \hline
\multicolumn{ 1}{|l|}{} & E1-E25 & \multicolumn{1}{r|}{1} & E1-E18 & \multicolumn{1}{r|}{0.01} \\ \cline{ 2- 5}
\multicolumn{ 1}{|l|}{} & E26-E50 & \multicolumn{1}{r|}{0.5} & E19-E30 & \multicolumn{1}{r|}{0.005} \\ \cline{ 2- 5}
\multicolumn{ 1}{|l|}{} & E51-E75 & \multicolumn{1}{r|}{0.25} & E31-E40 & \multicolumn{1}{r|}{0.001} \\ \cline{ 2- 5}
\multicolumn{ 1}{|l|}{} & E75-E100 & \multicolumn{1}{r|}{0.125} &  &  \\ \hline
\end{tabular}
\caption{The typical hyper-parameter setup when training CIFAR and ImageNet with AlexNet model. To achieve best possible accuracy, researchers conduct heavy dataset-specific hyper-parameter tuning (e.g., sophisticated learning rate adaption schemes). }
\label{tab:motivation}
\end{table}

\begin{table}[htbp]
\vspace{0.2cm}
\centering
\scriptsize
\begin{tabular}{|l|r|l|l|}
\hline
Dataset size & \multicolumn{1}{l|}{small} & medium & large \\ \hline
Mini-batch size & 2 & \multicolumn{1}{r|}{4} & \multicolumn{1}{r|}{32} \\ \hline
Number of epochs & \multicolumn{ 3}{r|}{200} \\ \hline
Learning rate & \multicolumn{ 3}{r|}{0.01} \\ \hline
\end{tabular}
\caption{The hyper-parameter setup in NLC. Instead of heavy tuning for each dataset as shown in Table~\ref{tab:motivation}, NLC uses a simpler and more conservative setup (e.g., small mini-batch size, small learning rate and large number of epochs) to satisfy all the users. Users are categorized in 3 groups based on the number of their training samples: small ($<$ 10K), medium (10K -- 100K), and large ($>$ 100K). 
}
\label{tab:nlc_setup}
\end{table}

In summary, TaaS workload characteristics study suggests that 
adopting the scale-out solutions such like distributed 
parameter server based frameworks is unsuitable for a set of industry deep 
learning tasks that require small mini-batch size and high frequency of model exchange.

\subsection{Theoretical Justification of Using Small Mini-batch}
In the previous section, we have empirically demonstrated training with large batch size can cause unacceptable accuracy loss. Based on a recent theoretical study \cite{liu-asgd-nips-2015}, we now justify that why training with small batch size can counter system staleness and is desired for distributed deep learning in general. 
\label{sec:smallbs_justification}
\begin{theorem} \cite{liu-asgd-nips-2015}
Under certain commonly used assumptions, if the learning rate is chosen in the optimal way and the staleness $\sigma$ is 
bounded by
\begin{equation}
\sigma \leq O\left(\sqrt{E/\mu^2}\right)
\label{eq:cond}
\end{equation}
where $E$ is the total number of epochs and $\mu$ is the mini-batch size, 
then the asynchronous parallel SGD algorithm converges in the rate
\[
O(\sqrt{1/E}).
\]
\end{theorem}
The result suggests the tolerance of the staleness $T$ relies on 
the mini-batch size $\mu$. First note that the convergence rate 
$O(\sqrt{1/E})$ is optimal. The prerequisite is that the staleness $\sigma$ 
is bounded by $O(\sqrt{E/\mu^2})$. The staleness $\sigma$ is 
usually propositional to the total number of learners. To 
satisfy the condition in (\ref{eq:cond}), either $\mu$ should be small 
enough or the epoch number $K$ should be large enough. In other words, 
given the number of learners and the total epoch number (or the total 
computational complexity), small mini-batch size is preferred. In addition, 
it also explains why small mini-batch size is potentially preferred even 
for SGD (running on a single worker), since SGD with mini-batch 
size $\mu$ can be considered as running Async-SGD with mini-batch 
size $1$ with $\mu$ workers (it implies that the staleness is $O(\mu)$). 
Thus, SGD with large mini-batch size is equivalent to Async-SGD with 
a small mini-batch size but a large staleness.

\section{Communication Bandwidth Requirement in TaaS}
\label{sec:modeling}
In this section, we measure the computation time for different mini-batch sizes(i.e. $\mu$)
over NLC workloads and public image classification 
workloads. We then calculate the minimum memory bandwidth 
requirement to achieve any speedup when ASGD protocol is employed. Finally, we demonstrate why none of the existing scale-out or scale-up solution can accelerate NLC workloads.

Each learner's execution loop consists of three components: $T_{train}$ (gradient 
calculation), $T_{pull}$ (pull weights), $T_{push}$(push gradients). 
Each parameter server's execution loop contains three components: $T_{receive}$ (receive gradients), $T_{apply}$ (apply weights update), and $T_{send}$ (send weights). When $\mu$ is large, $T_{train} + T_{pull} + T_{push} \gg T_{receive} + T_{apply} + T_{send}$; when $\mu$ is small, time spent on PS becomes the critical path.
In ASGD,  sending weights and receiving gradients operations may overlap. To achieve any speedup, we then must have $T_{train} \ge (T_{receive} + T_{apply})$ \footnote{(i)Apply update and receive gradients cannot overlap, since apply update can only start when gradients are fully received (ii) Assuming learner can push gradients and receive weights instantaneously}. Note that the apply update operation is memory-bound level 1 BLAS operation. Combined memory bandwidth between GPU and CPU (gradients transfer) and memory bandwidth used in CPU DRAM (weights update) are of the same order of magnitude. Further, gradients and weights are of the same size. We now can infer the required overall communication bandwidth to observe any speedup is at least $\frac{2 \times ModelSize}{TrainTime_{per minibatch}}$.  For NLC workload and image recognition workload, Table~\ref{tab:bandwidth_requirement} records Training time Per Epoch (TPE), Training Samples number ($N$), Model Size; and calculates the minimum Required Bandwidth (RB)  to observe any speedup.


\begin{table*}[htbp]
\scriptsize
\centering
\begin{tabular}{|r|r|r|r|r|r|r|r|r|r|r|r|r|}
\hline
\multicolumn{ 1}{|l|}{} & \multicolumn{ 2}{l|}{\textbf{\jewel}} & \multicolumn{ 2}{c|}{\textbf{\welltok}} & \multicolumn{ 2}{c|}{\textbf{Age}} & \multicolumn{ 2}{c|}{\textbf{Yelp}} & \multicolumn{ 2}{c|}{\textbf{CIFAR }} & \multicolumn{ 2}{c|}{\textbf{ImageNet}} \\ \cline{ 2- 13}
\multicolumn{ 1}{|l|}{} & \multicolumn{ 2}{l|}{\textbf{2.46K,7.69MB}} & \multicolumn{ 2}{c|}{\textbf{7K,20.72MB}} & \multicolumn{ 2}{c|}{\textbf{68.48K,72.86MB}} & \multicolumn{ 2}{c|}{\textbf{500K,98.60MB}} & \multicolumn{ 2}{c|}{\textbf{50K,59,97MB}} & \multicolumn{ 2}{c|}{\textbf{1280K,244.48MB}} \\ \cline{ 2- 13}
\multicolumn{ 1}{|l|}{\textbf{$\mu$}} & \multicolumn{1}{l|}{\textbf{TPE}} & \multicolumn{1}{l|}{\textbf{RB}} & \multicolumn{1}{l|}{\textbf{TPE}} & \multicolumn{1}{l|}{\textbf{RB}} & \multicolumn{1}{l|}{\textbf{TPE}} & \multicolumn{1}{l|}{\textbf{RB}} & \multicolumn{1}{l|}{\textbf{TPE}} & \multicolumn{1}{l|}{\textbf{RB}} & \multicolumn{1}{l|}{\textbf{TPE}} & \multicolumn{1}{l|}{\textbf{RB}} & \multicolumn{1}{l|}{\textbf{TPE}} & \multicolumn{1}{l|}{\textbf{RB}} \\ \cline{ 2- 13}
\multicolumn{ 1}{|l|}{} & \multicolumn{1}{l|}{\textbf{(sec)}} & \multicolumn{1}{l|}{\textbf{(GB/s)}} & \multicolumn{1}{l|}{\textbf{(sec)}} & \multicolumn{1}{l|}{\textbf{(GB/s)}} & \multicolumn{1}{l|}{\textbf{(sec)}} & \multicolumn{1}{l|}{\textbf{(GB/s)}} & \multicolumn{1}{l|}{\textbf{(sec)}} & \multicolumn{1}{l|}{\textbf{(GB/s)}} & \multicolumn{1}{l|}{\textbf{(sec)}} & \multicolumn{1}{l|}{\textbf{(GB/s)}} & \multicolumn{1}{l|}{\textbf{(sec)}} & \multicolumn{1}{l|}{\textbf{(GB/s)}} \\ \hline
1 & 5.61 & 7.00 & 25.22 & 11.50 & 574.80 & 17.36 & 4376.12 & 22.53 & \multicolumn{1}{l|}{N/A*} & \multicolumn{1}{l|}{N/A*} & \multicolumn{1}{l|}{N/A*} & \multicolumn{1}{l|}{N/A*} \\ \hline
2 & 3.22 & 6.10 & 14.11 & 10.28 & 309.51 & 16.12 & 2367.04 & 20.83 & 792.45 & 3.78 & 26957.54 & 11.61 \\ \hline
4 & 1.77 & 5.54 & 7.46 & 9.72 & 169.99 & 14.68 & 1299.42 & 18.97 & 502.63 & 2.98 & 15596.41 & 10.03 \\ \hline
8 & 1.00 & 4.89 & 4.06 & 8.93 & 104.39 & 11.95 & 802.04 & 15.37 & 356.46 & 2.10 & 9499.38 & 8.24 \\ \hline
16 & 0.75 & 3.27 & 2.79 & 6.49 & 77.79 & 8.02 & 571.10 & 10.79 & 290.95 & 1.29 & 7294.30 & 5.36 \\ \hline
32 & 0.72 & 1.69 & 2.38 & 3.82 & 69.52 & 4.49 & 451.61 & 6.82 & 261.23 & 0.72 & 5769.76 & 3.39 \\ \hline
64 & 0.80 & 0.76 & 2.37 & 1.91 & 77.14 & 2.02 & 403.69 & 3.82 & 245.68 & 0.38 & 5014.85 & 1.95 \\ \hline
128 & 1.00 & 0.31 & 2.77 & 0.82 & 100.22 & 0.78 & 402.97 & 1.91 & 234.38 & 0.20 & 4784.22 & 1.02 \\ \hline
\end{tabular}
\caption{Minimum bandwidth requirement to see any speedup. $\mu$ is mini-batch size, TPE stands for Time Per Epoch, RB stands for Required Bandwidth. *Both CIFAR and ImageNet use models that use Batch Normalization (BN), which requires $\mu \geq 2$.}
\label{tab:bandwidth_requirement}
\end{table*}

\underline{Why a scale-out solution will never work ?}
From Table~\ref{tab:bandwidth_requirement}, it is easy to see a 10GB/s bandwidth network is required to achieve any speedup for NLC workloads with the appropriate mini-batch size. In addition, to achieve $X$-fold ($X >$ 1) speedup, we need to multiply RB by a factor of $X$. Such a demanding bandwidth is beyond the capacity of advanced network techniques (e.g., RDMA). Note RB is also quite close to the peak memory bandwidth (e.g., PCI-e, DRAM), which indicates any extra memory copy may make speedup impossible. Thus, it is natural to infer that the only viable PS architecture is a tightly coupled multi-GPU system collocated on the same server that minimizes data copies and enables learners to asynchronously push gradients and pull weights. 

\underline{Why existing scale-up solutions are insufficient ?}
Among popular open-source deep learning frameworks, Caffe\cite{caffe}, Torch\cite{torch} and TensorFlow\cite{tensorflow} support multi-GPU training on the same node. However, they are designed for tasks where heavy hyper-parameter tuning is allowed so a larger mini-batch size may be appropriate (e.g., 256). It is easy to see from Table~\ref{tab:bandwidth_requirement} that it requires much higher communication bandwidth to support a small mini-batch than to support a large mini-batch.   
In addition, Caffe and TensorFlow only support SSGD on one node. We have demonstrated in Section~\ref{sec:nlc} that some of the workloads require mini-batch size to be as small as 2, which means Caffe and TensorFlow can at most make use of 2 GPUs (e.g., each GPU works with a mini-batch size of 1). Torch is the only open-source DL framework that supports both SSGD (via DPT) and ASGD (via mpiT) on a single-node. However, as demonstrated in Section~\ref{sec:eval-compare},  neither DPT nor mpiT can efficiently use the memory bandwidth on the same node. Furthermore, none of the existing solutions provide a fault-tolerance mechanism in the scale-up setting.

 \section{Design and Implementation}
\label{sec:design}
\subsection{Overall Design}


\Tool strives to minimize memory copy and enable high-concurrency to maximize communication bandwidth utilization. Figure~\ref{fig:design_async} depicts its design. To minimize memory copy, PS and learners use a shared-memory region to exchange gradients and weights. Each learner has a fixed number of slots in the producer-consumer queue, thus the entire system can be viewed as $\lambda$\footnote{$\lambda$ is the number of Learners.} single-producer-single-consumer queues. PS updates weights in place (i.e., HogWild! style). We use 4 \texttt{openmp} threads and unroll weights update loop 8 times to maximize DRAM throughput. To maximize system concurrency, each learner creates two additional threads -- push thread and pull thread. On the same GPU device, each learner maintains an on-device gradient staging buffer; so that after a learner finishes gradient calculation it can store the gradients in the buffer, and continue the next gradient calculation without waiting for the completion of push. On-device memory bandwidth is usually several hundreds of GB/s, which is much faster than device-host memory bandwidth (typically $\sim$10 GB/s). By buffering gradients on the same device, learner can train continuously while the push thread is pushing gradients to PS. Similarly, each learner also maintains a weights staging buffer on the same device. Learners do not communicate with each other, they only communicate with parameter server.  


Figure~\ref{fig:thread_code} details the necessary logic of PS and learners. We use the following naming conventions: variables that start with 'g\_' (e.g., g\_env, g\_param\_ptr) represent shared variables between PS and Learners, other variables (e.g., pullCnt, pushCnt) are shared variables between learner's main thread and its communication threads. PS (Figure~\ref{fig:ps_code}) iterates over the gradient queues in a round-robin fashion and it busy-loops when all the queues are empty. PS does not yield CPU via a conditional variable wait, because PS demands the most CPU cycles to process gradients and it is beneficial to have PS takeover gradients whenever they are ready.
Figure~\ref{fig:training_code} illustrates the logic of learner main thread. It calculates gradients on GPU's default stream. The learner main thread communicates with push thread (Figure~\ref{fig:push_code}) via a producer-consumer queue (Line 105 - 107 in Figure~\ref{fig:training_code} and corresponding Line 203-206, 220-222 in Figure~\ref{fig:push_code}) of size 1, i.e. variable \textit{pushCnt} in Figure~\ref{fig:training_code} and Figure~\ref{fig:push_code} alternates between 0 and 1. The push thread operates on a separate stream so that it can send gradients in concurrent with the learner thread calculating the gradients. Similarly the learner main thread communicates with the pull thread (Figure~\ref{fig:pull_code}) via a producer-consumer-queue (Line115-125 in Figure~\ref{fig:training_code} and in Figure~\ref{fig:pull_code}) of size 1, i.e., variable \textit{pullCnt} in Figure~\ref{fig:training_code} and Figure~\ref{fig:pull_code} alternates between 0 and 1. If the weights in the learner thread is current (i.e. has the same timestamp as the weights on PS), pullThd skips the pull request in this iteration.  By default, the PS updates weights at Line 7 in Figure~\ref{fig:ps_code} in a lock-free fashion (e.g., an incarnation of HogWild! algorithm\cite{hogwild}). \Tool also supports protecting weight updates from concurrent pulling via a read-write-lock.

Note that $cudaStreamSynchronize()$ invoked during the execution is to make sure the memory was flushed in place (either between host and device or within the same device) w.r.t the same stream. As a result, the corresponding memory copy between CPU/GPU or within GPU strictly follows programming order, which is necessary for our protocol verification, and which will be described in the next section. 

\begin{figure}[!htb]
  \centering
  \includegraphics[scale=0.25]{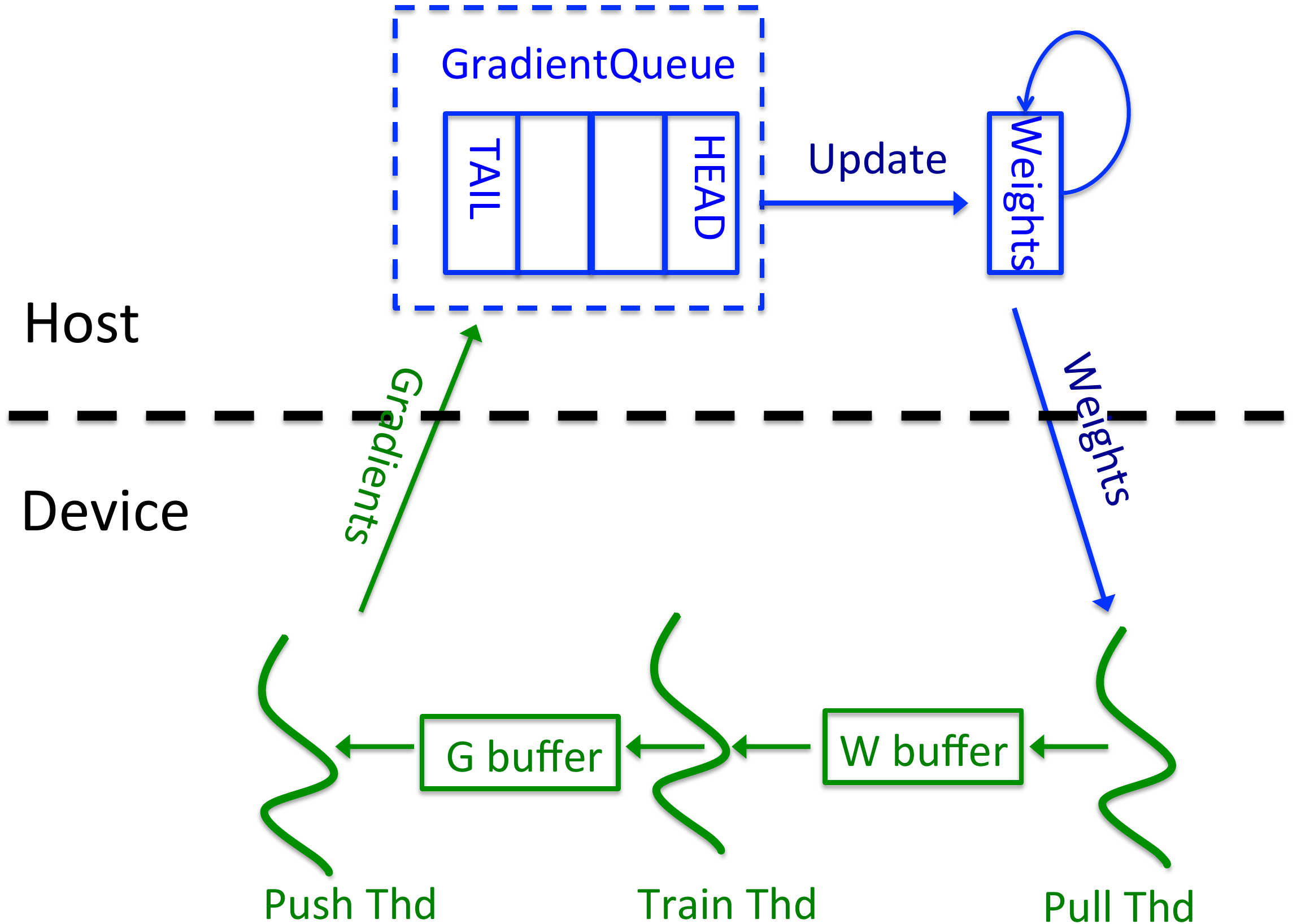} 
\caption{Overview of \Tool's Design}
\label{fig:design_async}
\end{figure}

\subsection{Verification of \Tool's communication protocol}
It is difficult to detect, avoid and fix concurrency bugs\cite{shan-asplos-2008}. \Tool relies on heavy communication between CPU threads, CPU-GPU interaction, and multi-stream operation within the same GPU device. It is imperative to verify the correctness of its communication protocol. In this section, we prove \Tool is deadlock-free in Theorem~\ref{theorem:deadlock-free} and verify \Tool's liveness property in Theorem~\ref{theorem:liveness}. 
\begin{figure*}[t]
  \centering
  \subfloat[{\small Host (Parameter Server Thread)}]{{\includegraphics[width=0.4\textwidth]{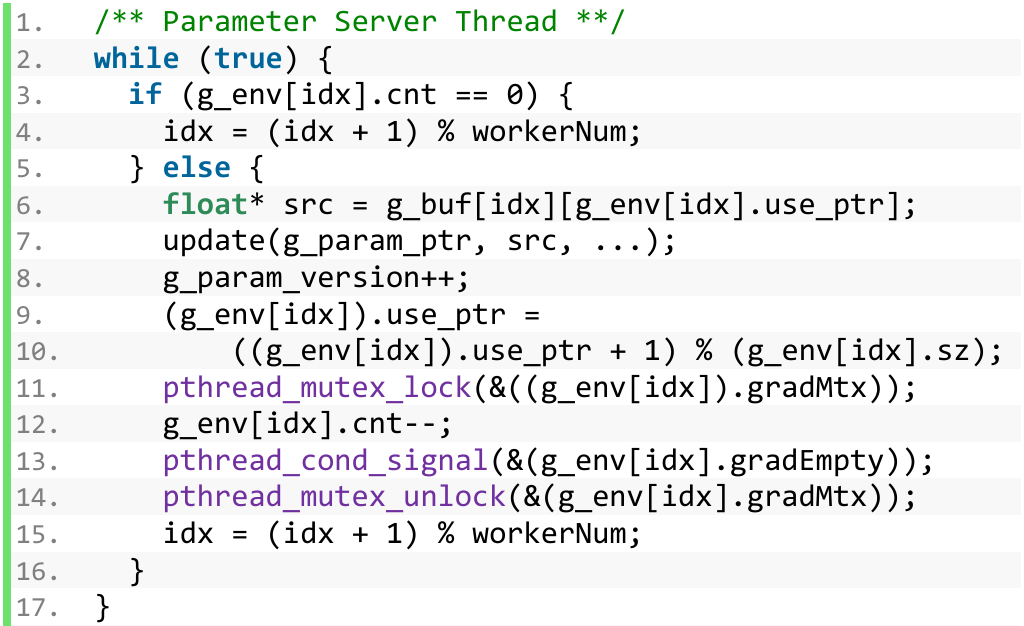}}\label{fig:ps_code}}
  \hspace{0.1in}
  \subfloat[\vspace{-1ex}{\small Pull Thread}]{{\includegraphics[width=0.4\textwidth]{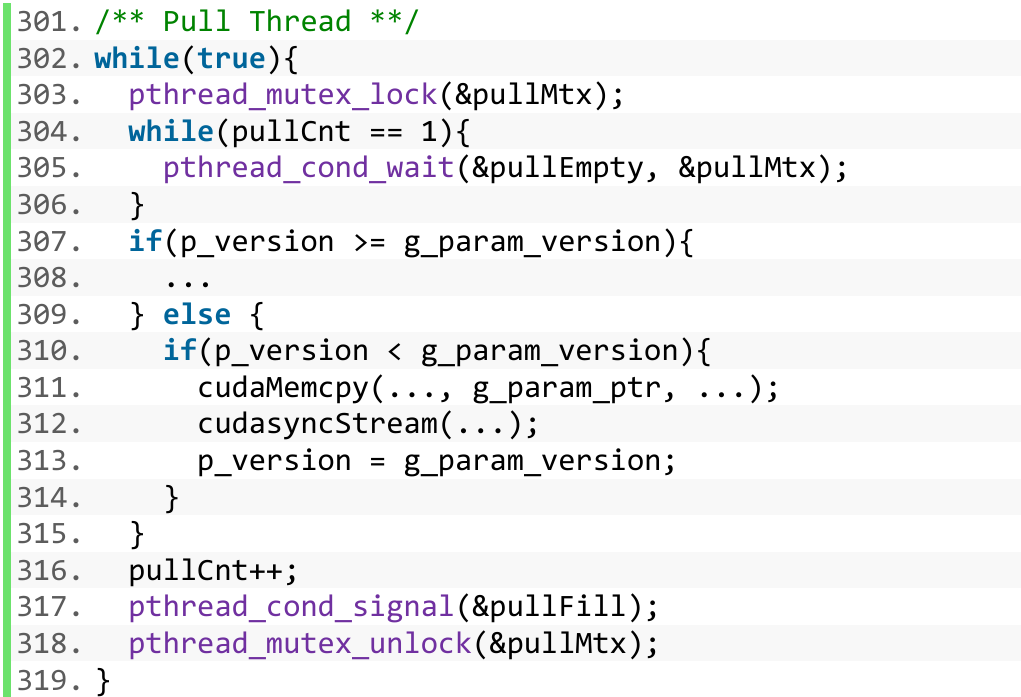}}\label{fig:pull_code}}\\[-0.5ex]   
  \subfloat[{\small Training Thread}]{{\includegraphics[width=0.4\textwidth]{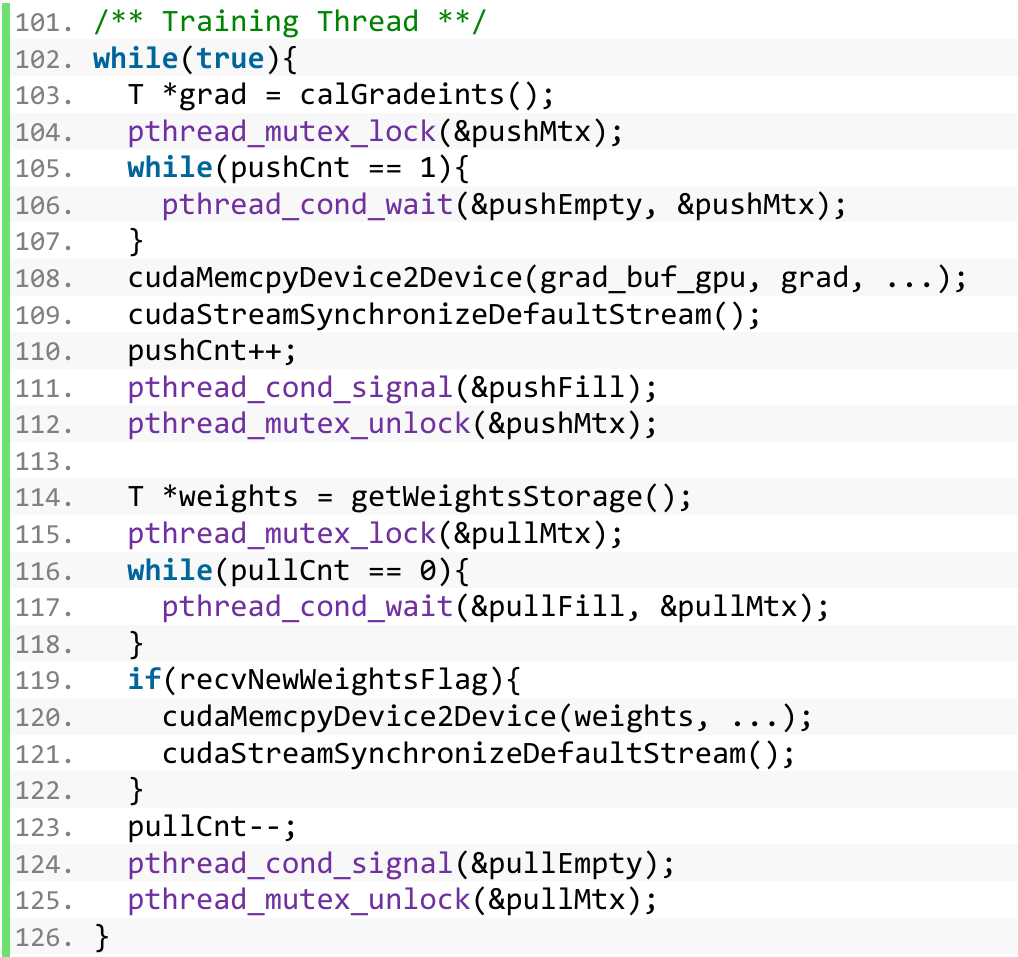}}\label{fig:training_code}}
  \hspace{0.1in}
  \subfloat[{\small Push Thread}]{{\includegraphics[width=0.4\textwidth]{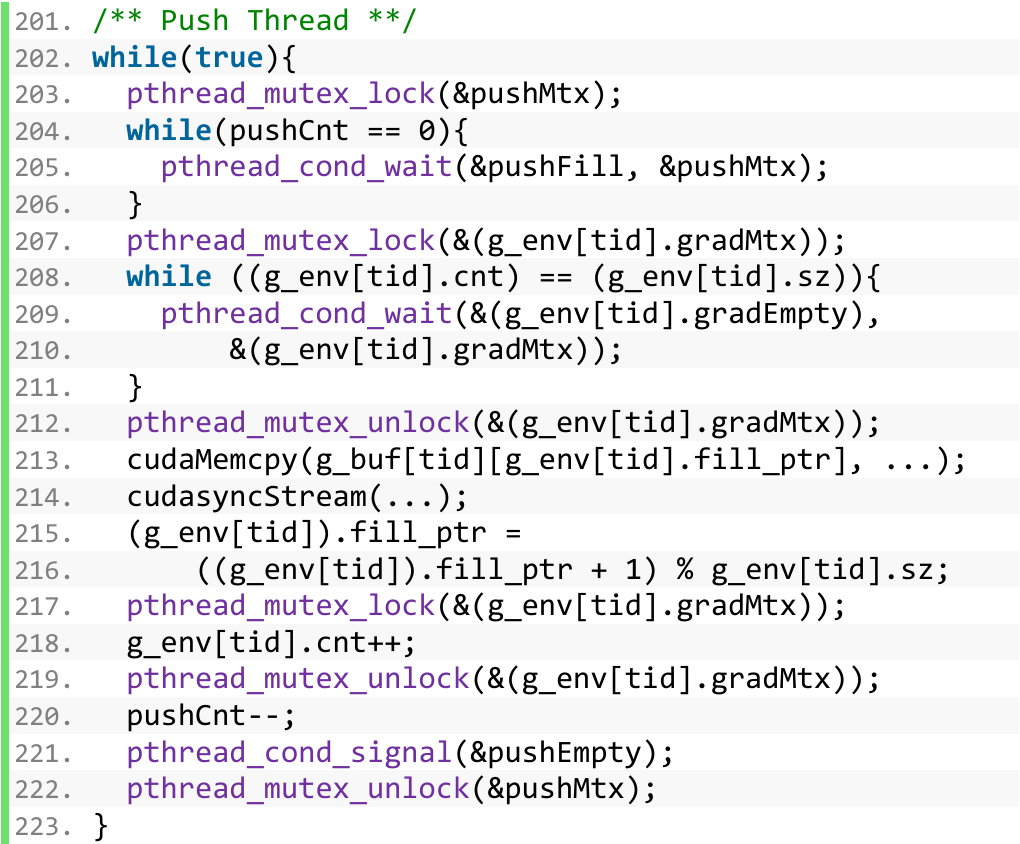}}\label{fig:push_code}}  
  \caption{\small Details of the parameter server, learner main thread (training thread), push thread and pull thread.}
  \label{fig:thread_code}
\end{figure*}

\begin{figure*}[t]
\centering
\includegraphics[width=0.8\textwidth]{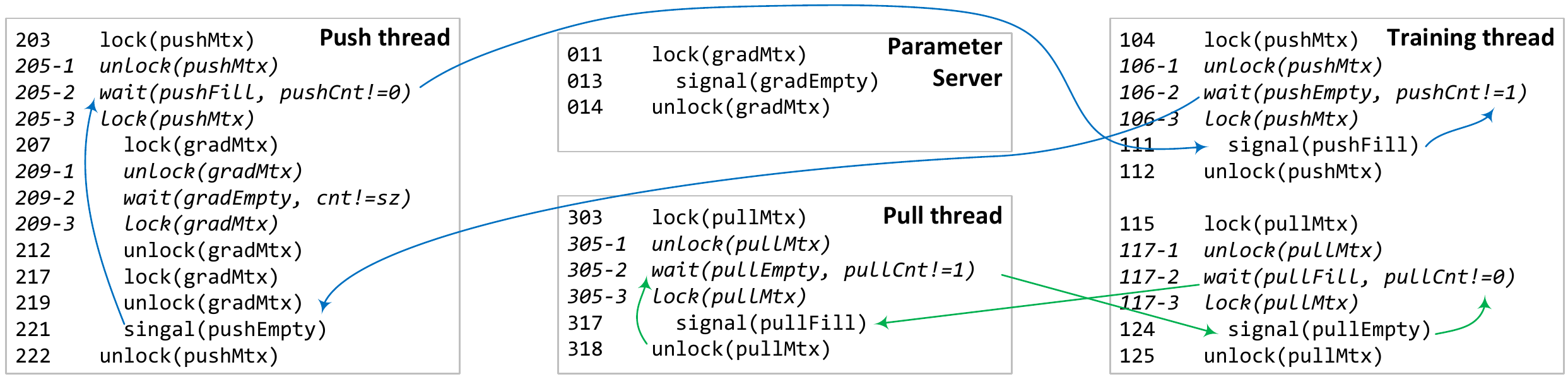}
\vspace{0.01in}
\caption{\small Synchronization operations extracted from Fig.~\ref{fig:thread_code}. 
The line numbers refer to the same lines in Fig.~\ref{fig:thread_code}. Lines in 
shape of $i$-$1$, $i$-$2$ and $i$-$3$ are equivalently transformed from the 
condition check loop at line $i$ in Fig.~\ref{fig:thread_code}. Operation 
\textit{wait(s, c)} means the thread is blocked until it's waken up by signal 
\textit{s} and condition \textit{c} is true.} \label{fig:wfg}
\end{figure*}

\begin{lemma}\label{lemma1}
  The value of \textsf{pushCnt} and \textsf{pullCnt} can only be 0 or 1. For the $idx$-th learner, $0 \leq$ \textsf{g\_env[idx].cnt} $\leq$ \textsf{g\_env[idx].sz}. \end{lemma}

\begin{proof}
\vspace{-0.03in}
We will show that \textsf{pushCnt} and \textsf{pullCnt} cannot be larger than 
$1$ or smaller than $0$. Line $110$ (Fig.~\ref{fig:training_code}) is the only 
place where \textsf{pushCnt} can be incremented. If \textsf{pushCnt} can be 
larger than 1 then there must be an 
iteration where \textsf{pushCnt} is $1$ before executing line $110$, since the increment is 1 per iteration. However, 
this is impossible because line $110$ is not reachable due to the condition 
check loop at line $105$. Similarly, line $220$ (Fig.~\ref{fig:push_code}) is 
the only place where \textsf{pushCnt} can be decremented. It cannot 
be smaller than $0$ due to the loop at line $204$.  Therefore, \textsf{pushCnt} 
can only be $0$ or $1$. The claims about \textsf{pullCnt} and \textsf{g\_env[idx].cnt} 
can be proved in the same way.
\end{proof}

\begin{lemma}
  Once signaled, a thread blocked by a condition wait (line 106, 117, 205, 209 
or 305 in Fig.~\ref{fig:thread_code}) will wake up and exit the 
corresponding condition check loop. 
\end{lemma}

\begin{proof}
\vspace{-0.03in}
We will show the loop conditions do not hold when signals are sent. For example, only line 
$222$ (Fig.~\ref{fig:push_code}) can wake up the condition wait on 
\textsf{pushEmpty} at line $106$ . According to Lemma \ref{lemma1}, 
\textsf{pushCnt} can only be $0$ or $1$. When sending the signal, 
\textsf{pushCnt} is always $0$ (due to line $220$) and thus invalidates the 
loop condition at line $105$. Therefore, the training 
thread will exit the condition check loop. Similarly, we can prove the claim is 
true for other condition waits.
\end{proof}

\begin{lemma} \label{lemma_wait}
 The wait-for graph formed by condition waits only, i.e. 
\textsf{pthread\_cond\_wait}, is acyclic and thus deadlock-free. \end{lemma}

\begin{proof}
\vspace{-0.03in}
  The directed edges in Fig.~\ref{fig:wfg} represent the wait-for relations among 
  wait and signal statements. The edges in blue and green form two cycles. We will show some edges in a 
  cycle cannot exist at the same time. 
  
  In the blue cycle, the edge from line $205$-$2$ to $111$ represents that the 
push thread waits until \textsf{pushCnt != 0}. The blue edge from $106$-$2$ to 
$221$ indicates the training thread waits until \textsf{pushCnt != 1}.  Given 
that \textsf{pushCnt} can only be  $0$ or $1$ (Lemma \ref{lemma1}), the above 
two conditions cannot be true at the same time. Therefore, these two edges 
cannot exist together and the cycle in blue is infeasible. In the green cycle, 
the edge from $305$-$2$ to $124$ illustrates the pull thread waits until 
\textsf{pullCnt = 0}.  The edge from $117$-$2$ to $317$ indicates the training 
thread waits until \textsf{pullCnt = 1}. Similarly, this cycle is infeasible 
too.
\end{proof}

\begin{lemma} \label{lemma_locks}
Mutex lock operations in Fig.~\ref{fig:thread_code} are deadlock-free.
\end{lemma}

\begin{proof}
\vspace{-0.03in}
Since the condition check loops (lines $105$-$107$, $116$-$118$, $204$-$206$, $208$-$211$ 
and $304$-$306$ in Fig.~\ref{fig:thread_code}) do not have side effects, we 
equivalently rewrite the condition check loops and summarize synchronization 
operations in Fig.~\ref{fig:wfg}. 

Now consider operations based on mutex locks in Fig.~\ref{fig:wfg}. One necessary conditions for deadlocks is {\it 
hold-and-wait} \cite{CoffmanCondition}. If threads are not holding one resource 
while waiting for another, there is no deadlock. However, only the 
push thread can be in a hold-and-wait state (lines $207$, $209$-$3$ and $217$). 
Therefore, mutex lock operations cannot introduce deadlocks.
\end{proof}

\begin{theorem}
\label{theorem:deadlock-free}
\Tool is dealock-free. 
\end{theorem}

\begin{proof}
\vspace{-0.03in} 
Lemma \ref{lemma_wait} and \ref{lemma_locks} show that neither condition waits 
nor mutex lock operations can cause deadlocks. Now we consider them 
together.
In Fig.~\ref{fig:wfg}, condition waits at $106$-$2$, $205$-$2$, $305$-$2$  and 
$117$-$2$ are not in any atomic region and thus are isolated from mutex locks. 
Hence, they cannot introduce deadlocks. The only remaining case is the condition 
wait at $209$-$2$. It is guarded by lock \textsf{pushMtx} and waits for the 
signal sent at $013$. However, line $13$ is protected by a different lock 
\textsf{gradMtx}. This doesn not satisfy the hold-and-wait condition. Therefore, 
it cannot lead to deadlocks either.
\end{proof}

\begin{theorem}
\label{theorem:liveness}
 The parameter server thread processes each gradient exactly once.
\end{theorem}

\begin{proof}
\vspace{-0.03in}
The push thread of learner \textsf{idx} shares gradients with the parameter 
server using a share array \textsf{g\_env[idx]}. Essentially, it is an 
array-based FIFO queue, where \textsf{g\_env[idx].cnt} indicates the number of 
gradients in queue. It is straightforward to see {\it neither the push thread nor 
the parameter server can access the same memory location in two consecutive 
iterations}. 

Similar to the proof for Lemma \ref{lemma1}, we can show that $0 \le g\_env[idx].cnt \le 
g\_env[idx].sz$ so that $g\_env[idx].sz$ is the max size of the queue. In 
addition, as the parameter server only processes a learner's gradients if its 
queue is not empty (lines $6$-$14$), the read pointer ($g\_env[idx].use\_ptr$) can never be ahead of 
the write pointer ($g\_env[tid].fill\_ptr$). So, {\it the parameter server cannot read an outdated gradient and use it more 
than once}. Similarly, when the queue reaches its max size, the push thread must 
wait (lines $208$-$211$). So, {\it the push thread cannot overwrite unprocessed 
gradients in the queue}.
\end{proof}

\subsection{Fault-tolerance}
\label{sec:design-ft}
In \Tool, each learner and the PS has its own address space, learners and PS communicate via \texttt{mmap}-ed memory. This approach is similar to Grace\cite{grace-oopsla}, which transforms multi-threaded program to multi-process program communicating via shared memory.  When \Tool starts, learners and PS \texttt{mmap} the same memory file, which pre-allocated gradient queues, shared weights, and thread related synchronization variables (e.g., mutexes and condition variables, both with \texttt{PTHREAD\_PROCESS\_SHARED} attributes set).  

PS periodically communicates with the watchdog process to log its progress and checkpoint the parameters. 
When $n$ learners unexpectedly die ($ n< \lambda $, $\lambda$ is the number of learners), PS continues to process gradients collected from alive learners so that failures from the dead learners are naturally isolated. Note PS is stateless in that it only needs to process a fixed number of gradients without considering which learners produced the gradients. \\
If all learners die, PS no longer makes progress, thus the watchdog process kills the PS and restarts PS and learners from the last checkpoint. If a learner dies when holding a lock, PS will hang when it tries to grab the lock. Thus the watchdog process will later detect the failure and take action. Alternatively, one may set the \texttt{robust} attribute
of pthread mutex so that when PS is grabbing the lock it can notice the failed learner and skip checking that learners' gradient queue.   
\subsection{Additional Design Decisions}

\textit{Computation optimization inside GaDei.} The major
  computation in GaDei's PS server is element-wise vector addition for
  updating the global weights. We unroll the weights update loop on the parameter server side 8
  times in the OpenMP parallel section, and found that it outperforms the
  non-unrolled version by 30\%. In addition, it outperforms the
  version in which we directly use the Streaming SIMD Extension
  (SSE) instructions.\\
\textit{Half-precision floating point operations.} Recent
  research work\cite{icml2015_gupta15, han-iclr} demonstrate that it
  is possible to train DNNs with lower precision and still obtain
  comparable model accuracy. Comparing to 32-bit single precision, the
  16-bit half-precision data format will decrease parameter server
  processing burden. The latest GPU architecture may support
  half-precision 16-bit float point operation, while general purpose
  CPUs do not. We integrated software-based 16-bit CPU floating point
  operations in GaDei. However, It slowed down the computation
  by a factor of 5, which makes any speedup impossible for
  our workloads.
  \\
  \textit{Lock-free producer-consumer queue.} A
  lock-free producer-consumer queue was considered in our design. Lock-free design
  can minimize the inter-process/thread latency, but would consume a
  CPU core in 100\% for each learner. As a result, the computation
  threads have to compete with the lock-free communication threads,
  and this will impede the computation performance
  significantly. In GaDei, when learners have to wait for the server
  to process gradients, yielding CPUs via conditional variable wait
  can save CPU cycles for PS computation threads to process gradients.


\begin{table*}[thbp]
\scriptsize
\centering
\begin{tabular}{|c|r|r|r|p{6cm}|r|r|}
\hline
 &  \textbf{Data Size} & \textbf{Label Size} & \textbf{Type} & \textbf{Description} & \textbf{Model Size} & \textbf{Network} \\ \hline
\textbf{\jewel} & 2.4k & 311  & NLC & Question answering task in insurance domain; label represents answers. & 7.69MB & CNN \\ \hline
\textbf{\welltok} & 7k & 3595  & NLC & Question answering task in online service domain; label represents answers. & 20.72MB & CNN\\ \hline
\textbf{Yelp} & 500k & 5  & NLC & Customer review classification; label represents the star the customer assigns to the business. & 98.60MB & CNN \\ \hline
\textbf{Age} & 68k & 5 & NLC & Author profiling task; the label represents the age range of the author. & 72.86MB & CNN \\ \hline
\textbf{MR} & 8.6k & 2 & NLC & Sentiment analysis of movie reviews; label represents positive/negative attitude of the audience.   &  14.27MB & CNN \\ \hline
\textbf{CIFAR} & 50k & 10 & IR & Classify images into  predefined categories. & 59.97MB & VGG \\\hline
\textbf{ImageNet} & 1280k & 1000  & IR & Classify images into  predefined categories. & 244.48MB & AlexNet  \\\hline
\end{tabular}
\caption{\small Task description and corpus statistics. NLC: natural language classification; IR: image recognition. CNN is the convolutional neural networks depicted in Figure \ref{fig:nlc_model}. AlexNet \cite{alexnet} and VGG \cite{vgg} are the de facto standard models in the field of image recognition.}
\label{tab:task_description}
\end{table*}

\section{Methodology}
\label{sec:meth}
\subsection{Software}

We use open source toolkit Torch \cite{torch} as the building block to calculate the gradients of neural nets. Torch is a scientific computing framework based on Lua/LuaJIT with both CPU and CUDA backends. 
Torch has been widely used both in industry (e.g., Facebook, Google Deepmind, and IBM) and in academic community. Researchers at Google\cite{tensorflow}, Bosch\cite{dl-framework-comparison}, and Facebook\cite{convnet-bench} have benchmarked several commonly used open-source deep learning frameworks (e.g. Caffe, Theano, Torch, and Tensorflow) and found that Torch usually outperforms other frameworks. In addition, Torch is the only framework that supports ASGD on one-node. \textbf{Thus to enable further speedup on top of Torch presents a bigger challenge for the system design.}

\subsection{Benchmark}
We use four representative NLC datasets in our evaluation. Two datasets \jewel and \welltok are in-house customer datasets. The other two datasets Age\cite{age} and Yelp\cite{yelp} are publicly available datasets. The \jewel and \welltok datasets represent the typical small datasets present in NLC workloads. Age dataset represents the medium-size dataset. Yelp dataset represents the large-size dataset. The learning rate is fixed at 0.01 across workloads, \Tool trains each NLC workload for 200 epochs. The neural network model used for the four datasets is presented in Figure \ref{fig:nlc_model}. To demonstrate that \Tool can be used as a drop-in replacement for tools that solve non-TaaS tasks, we also conducted experiments on image recognition tasks CIFAR \cite{Krizhevsky09learningmultiple} and ImageNet \cite{ILSVRC15}. We train CIFAR using VGG model and we train ImageNet using AlexNet model. We use the widely-adopted hyper-parameter setup, as described in \cite{vgg_cifar, alexnet_imagenet}, to train CIFAR and ImageNet tasks to demonstrate no accuracy loss is incurred.
 Table \ref{tab:task_description} records the task description and training data statistics.


\subsection{Hardware}
The experiments have been conducted on the Softlayer cloud\footnote{http://www.softlayer.com}.
The server is equipped with two Intel Xeon E5-2690-V3 processors. Each processor has 12 real cores, clocked at 2.66GHz per core. To enable the best possible PS CPU processing speed, we turn off SMT.  The CPU memory capacity is 128GB, with peak memory bandwidth ~40GB/s. There are two NVIDIA Tesla K80s installed on the server. Each K80 has two GPUs. Totally there are four GPUs on the server with a total ~16 TFlops. The bus interface of K80 is PCIe 3.0 x16, with a 12Gbps bi-directional bandwidth each lane.



\section{Experiment Results}
\label{sec:results}
In this section, we evaluate \Tool's ability to achieve good model accuracy and runtime performance, against widely used state-of-the-art open-source multi-GPU implementation DPT and mpiT on both commercial workloads and public workloads. Section~\ref{sec:eval-correctness} demonstrates how the model accuracy progresses w.r.t training epoch. Section~\ref{sec:eval-speedup} illustrates the runtime performance of \Tool . Section~\ref{sec:eval-compare} illustrates how \Tool achieves good speedup on challenging NLC workload, while other state-of-the-art tools cannot achieve any speedup. \Tool can achieve speedup using much smaller mini-batch size than any other tools. Section~\ref{sec:eval-ft} discusses \Tool's ability to handle fault-tolerance and GPU over-subscription.

\subsection{Convergence result}
\label{sec:eval-correctness}
\begin{figure*}[t]
\small
    \centering
    \subfloat[{\small \jewel accuracy}]{{ \includegraphics[trim={0.2cm 0.2cm 0.2cm 0.2cm},clip,width=0.32\textwidth, scale=0.6]{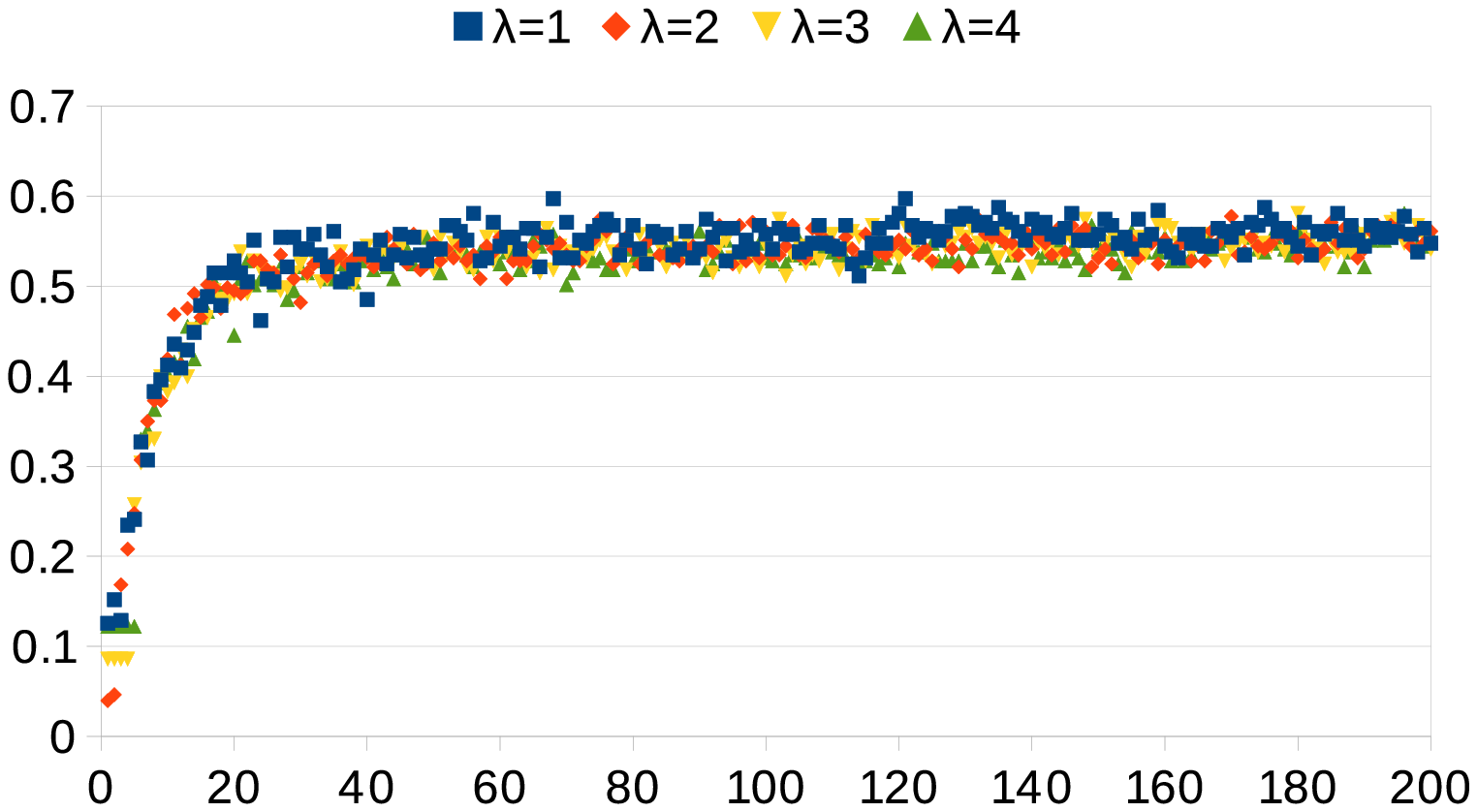} }\label{fig:corr_jewel}}\hfill
    \subfloat[{\small \welltok accuracy}]{{ \includegraphics[trim={0.2cm 0.2cm 0.2cm 0.2cm},clip,width=0.32\textwidth, scale=0.7]{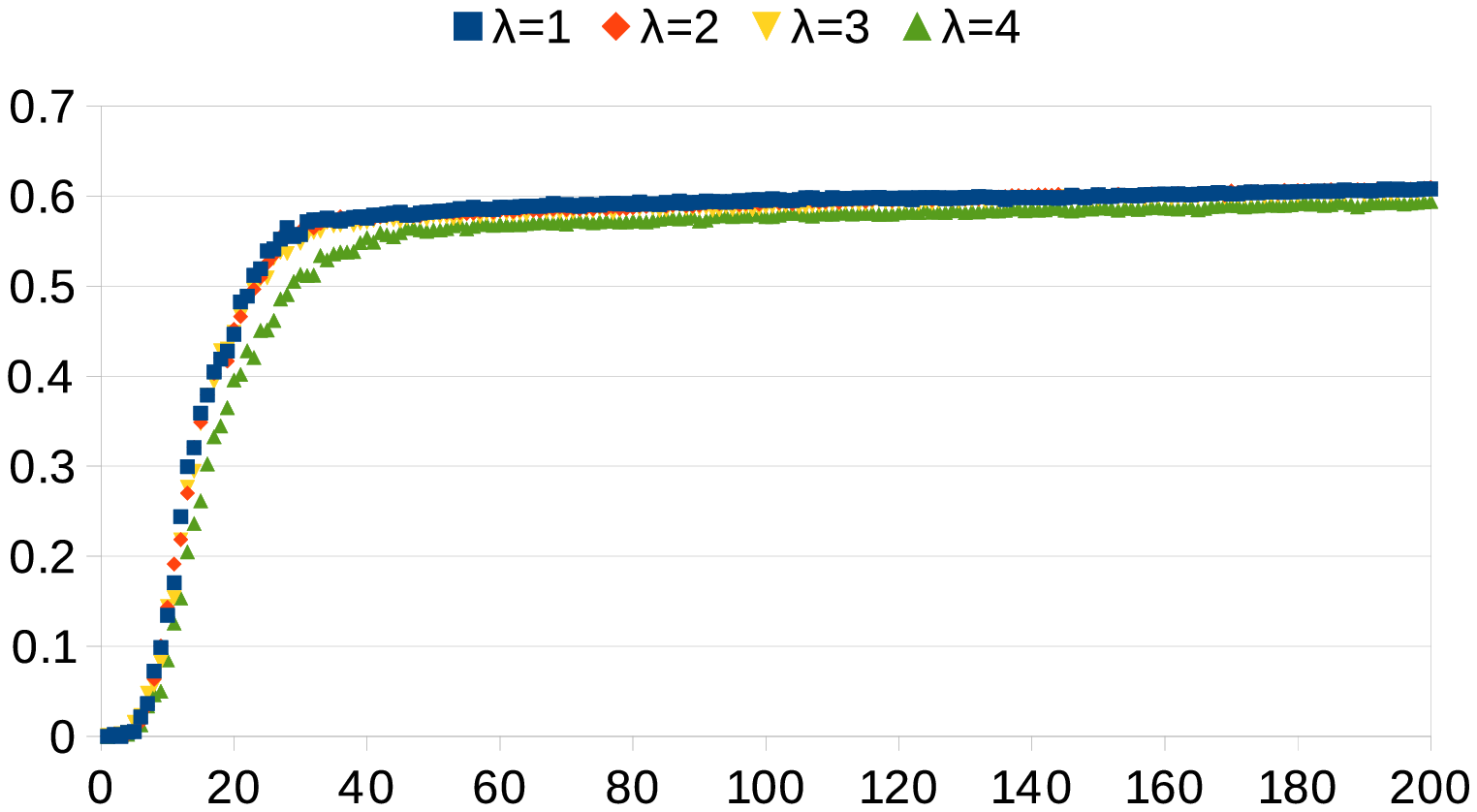} }\label{fig:corr_welltok}}\hfill
    \subfloat[{\small Age accuracy}]{{ \includegraphics[trim={0.2cm 0.2cm 0.2cm 0.2cm},clip,width=0.32\textwidth, scale=0.7]{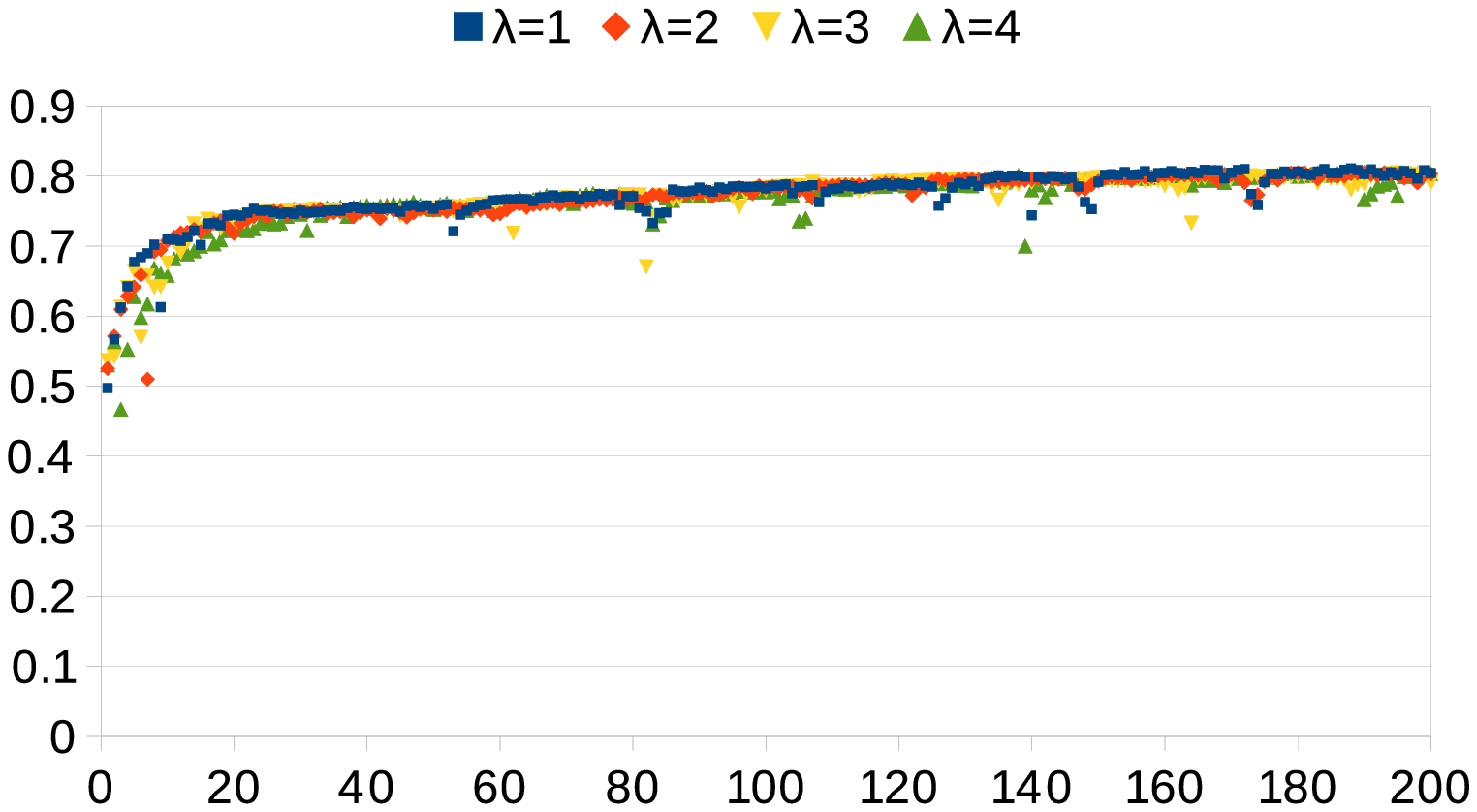} }\label{fig:corr_age}}\\[-2ex] 
    \subfloat[{\small Yelp accuracy}]{{ \includegraphics[trim={0.2cm 0.2cm 0.2cm 0.2cm},clip,width=0.32\textwidth, scale=0.7]{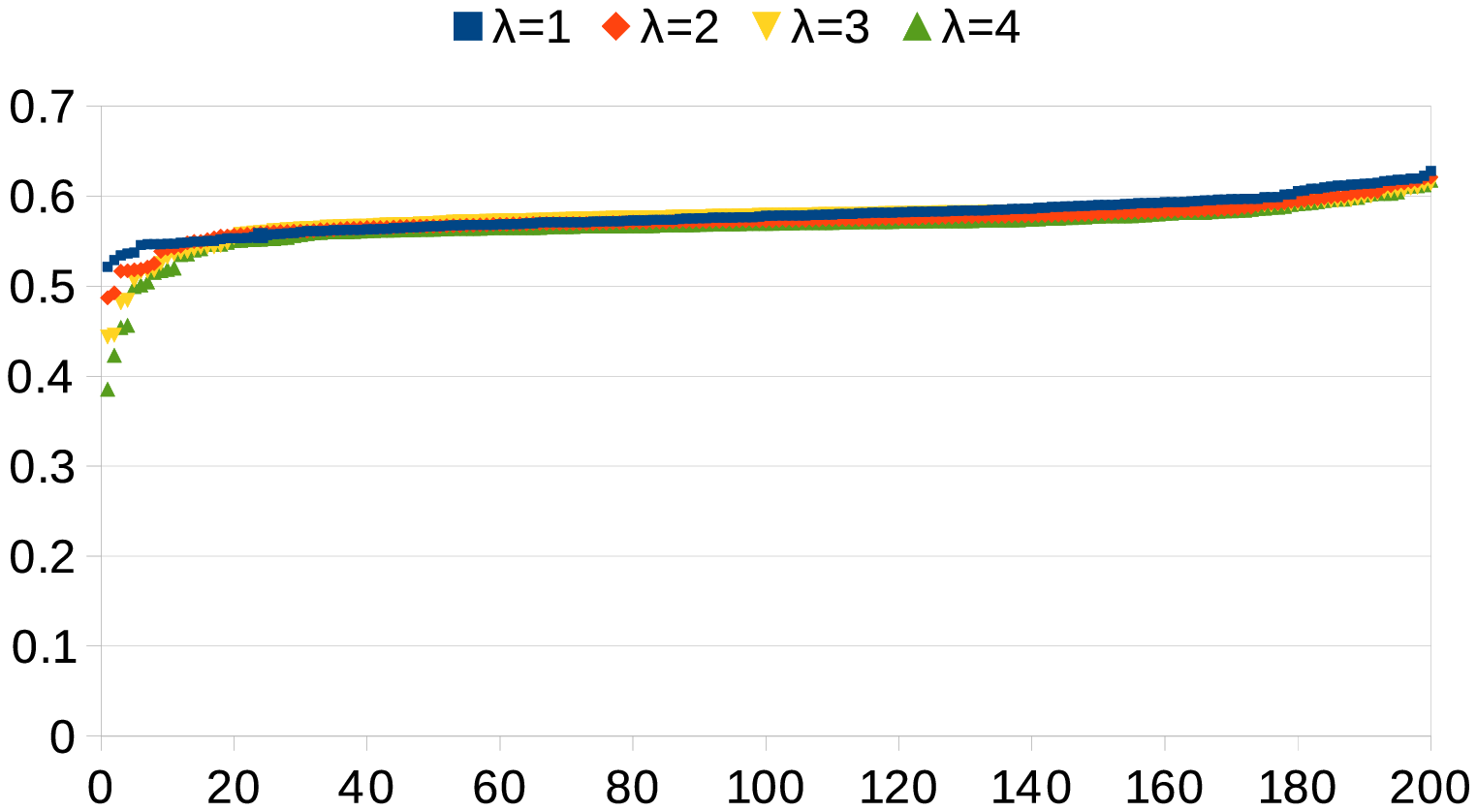} }\label{fig:corr_yelp}}\hfill
        \subfloat[{\small CIFAR accuracy}]{{ \includegraphics[trim={0.2 0.2cm 0.2cm 0.2cm},clip,width=0.32\textwidth, scale=0.7]{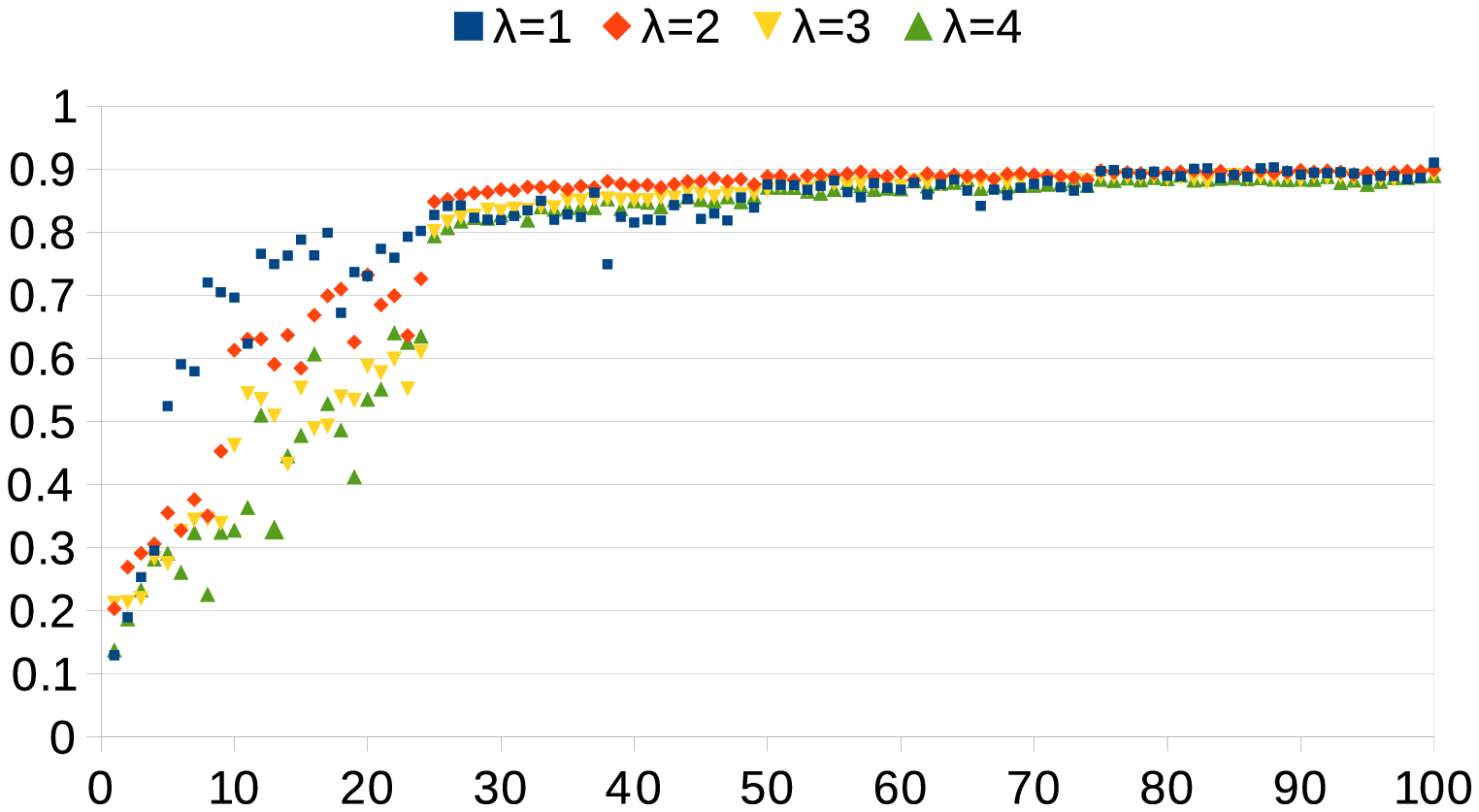} }\label{fig:corr_cifar}}\hfill
        \subfloat[{\small ImageNet accuracy}]{{ \includegraphics[trim={0.2 0.2cm 0.2cm 0.2cm},clip,width=0.32\textwidth, scale=0.7]{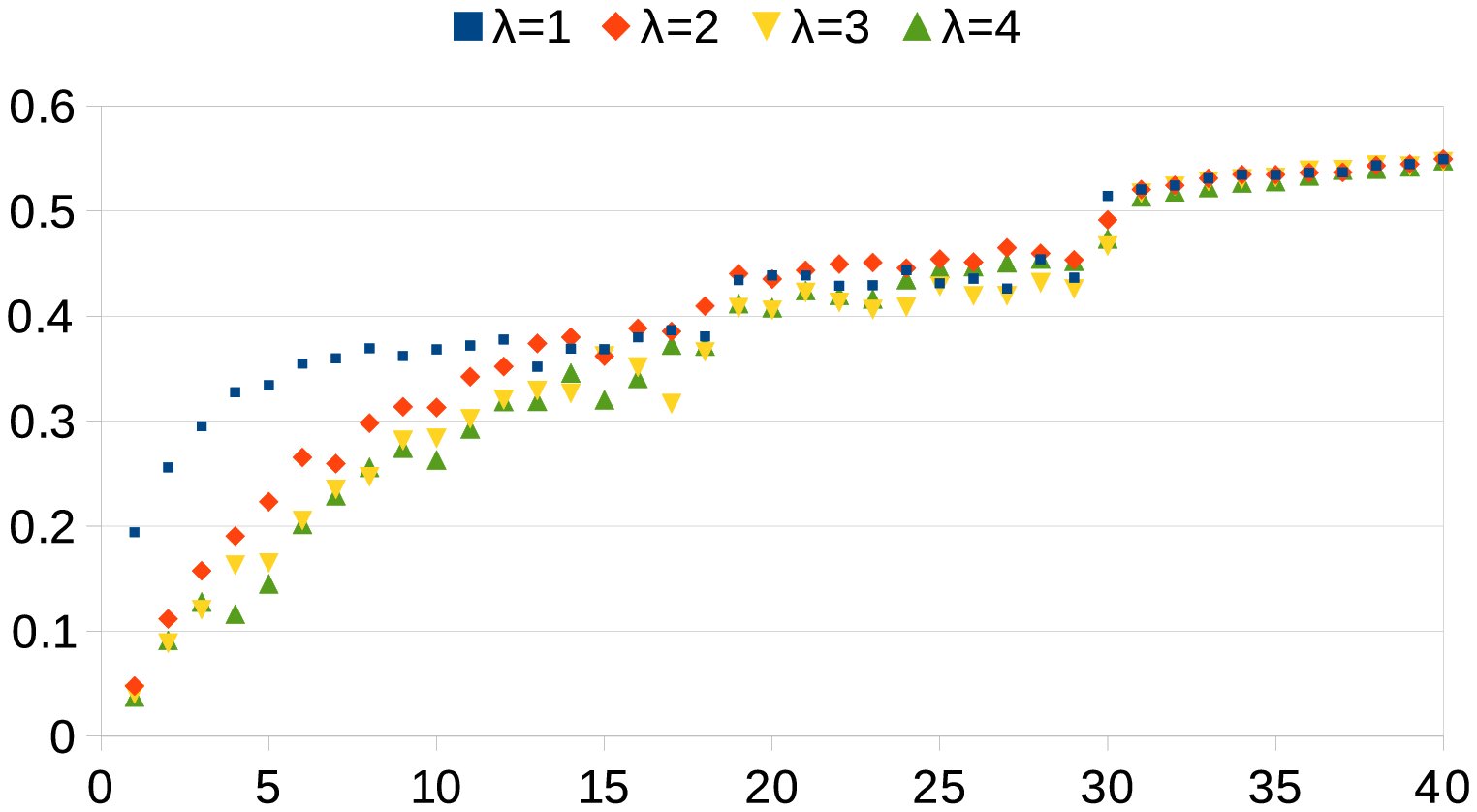} }\label{fig:corr_imgnet}}\\[-1ex]
    
\caption{\small Model accuracy w.r.t the training epochs when using 1,2,3,4 GPUs. \jewel and \welltok converge to 60\% , Yelp  62\%, Age 80\%, CIFAR  90\%, ImageNet 55\%. \Tool does not introduce accuracy loss which satisfies our expectation.}
    \label{fig:correctness_gadei}
\end{figure*}
In Figure~\ref{fig:correctness_gadei} we plot the model accuracy w.r.t the training epochs when using 1,2,3,4 learners.\jewel and\welltok converge to 60\% , Age 80\%, Yelp 62\%, CIFAR 90\%, ImageNet 55\%. Model accuracy reaches the same level ($\pm 1\%$) of accuracy as the single learner system within the same number of training epochs.
\vspace{-0.2cm}
\paragraph{Summary}
\Tool converges to the same level of accuracy as the single-learner SGD using the same number of epochs. This demonstrates a tightly-coupled system such as \Tool can mostly avoid the staleness issue introduced in a typical ASGD system, when using small minibatch size.

\subsection{Speedup Results and Memory Bandwidth Analysis}
\begin{table}[th]
\scriptsize
\centering
    \begin{tabular}{|l|r|r|r|}\hline
    \textbf{Workload} &  \textbf{Model Size}   & \textbf{Epochs}  & \textbf{Running Time}  \\ 
             &  \textbf{(MB)}         &         & \textbf{(hrs)}       \\
    \hline
    \textbf{\jewel}    &  7.69             &  200    & 0.18 \\
    \hline
    \textbf{\welltok}  & 20.72             &  200    & 0.78  \\
    \hline
    \textbf{Age}      & 72.86             &  200    & 9.44 \\
    \hline
    \textbf{Yelp}     & 98.60             &  200    & 24.99 \\
    \hline
    \textbf{Cifar}    & 59.97             &  100    & 6.53 \\
    \hline
    \textbf{ImageNet} & 244.48            &  40    & 52.56 \\ \hline
   \end{tabular}
    \caption{\small Single-learner performance baseline on one GPU.}
    \label{tab:perf-baseline}
\end{table}

\label{sec:eval-speedup}
\begin{figure*}
   \small
    \centering
    \subfloat[{\small \jewel speedup}]{{\centering \includegraphics[trim={0 0.2cm 0.3cm 0.3cm},clip, width=0.32\textwidth, height=2.6cm]{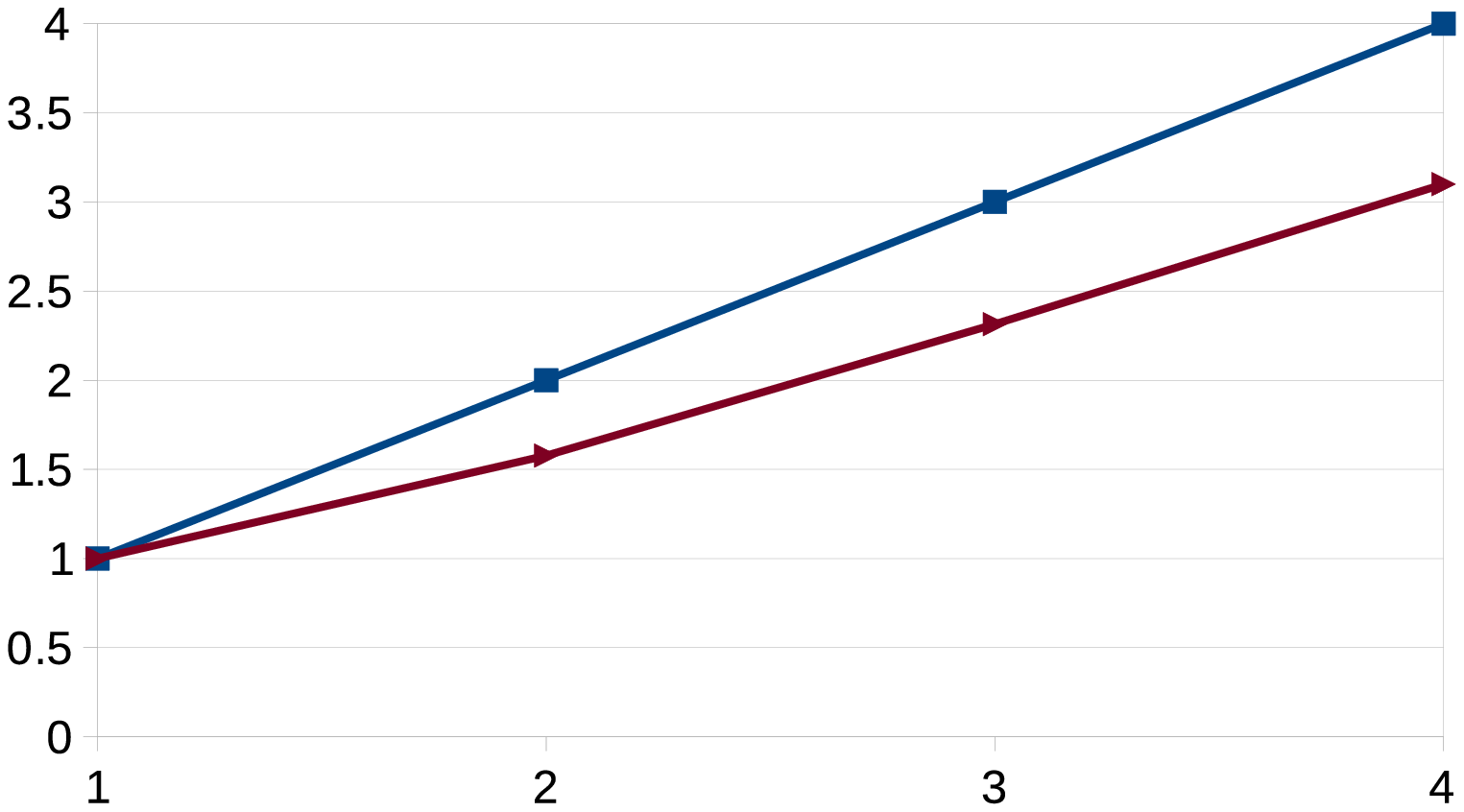} }\label{fig:speedup_jewel}}\hfill
    \subfloat[{\small \welltok speedup}]{{ \includegraphics[trim={0 0.2cm 0.3cm 0.3cm},clip,width=0.32\textwidth, height=2.6cm]{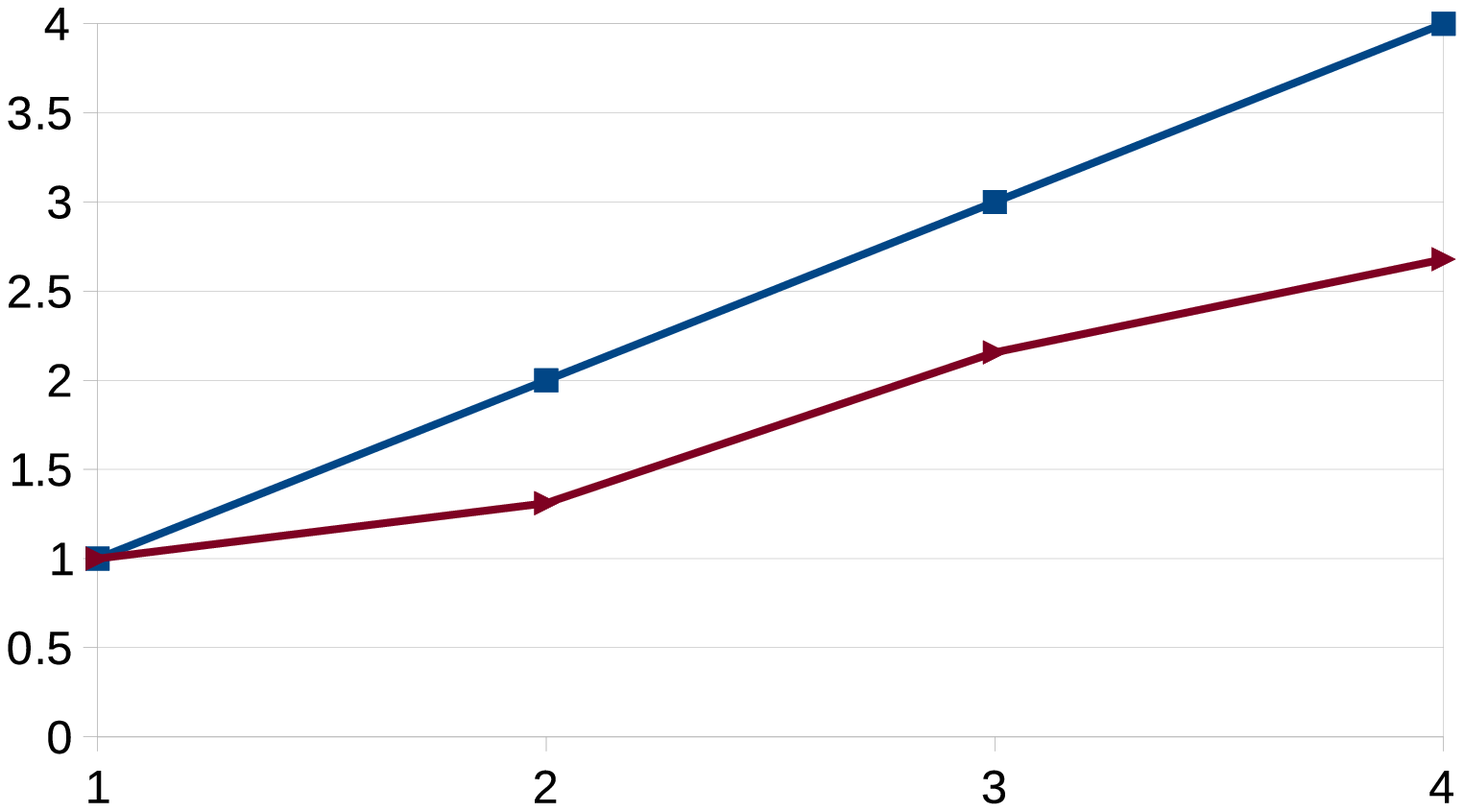} }\label{fig:speedup_welltok}}\hfill
    \subfloat[{\small Age speedup}]{{ \includegraphics[trim={0 0.2cm 0.3cm 0.3cm},clip,width=0.32\textwidth ,height=2.6cm]{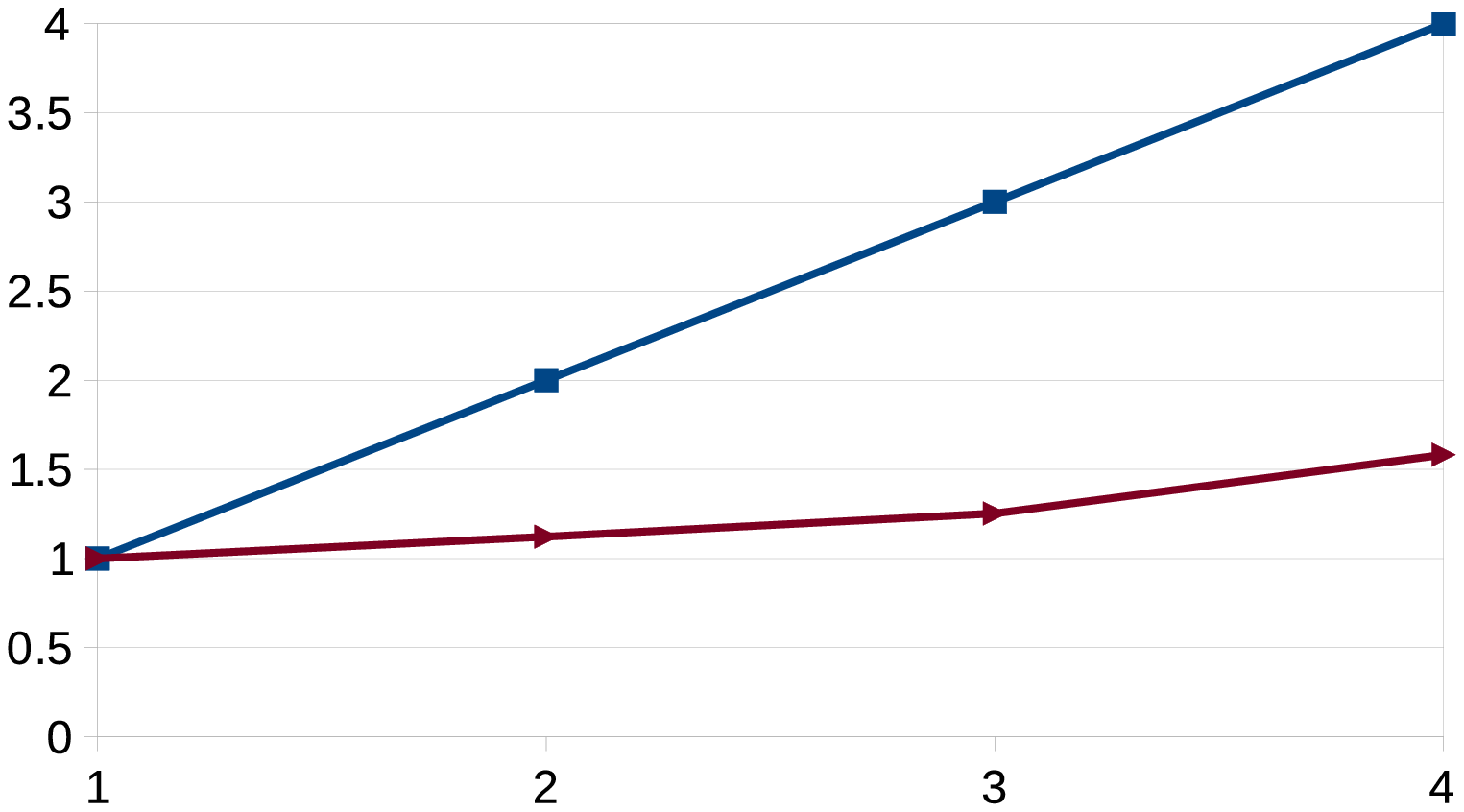} }\label{fig:speedup_age}}\\[-2ex]
    \subfloat[{\small Yelp speedup}]{{ \includegraphics[trim={0 0.2cm 0.3cm 0.3cm},clip,width=0.32\textwidth ,height=2.6cm]{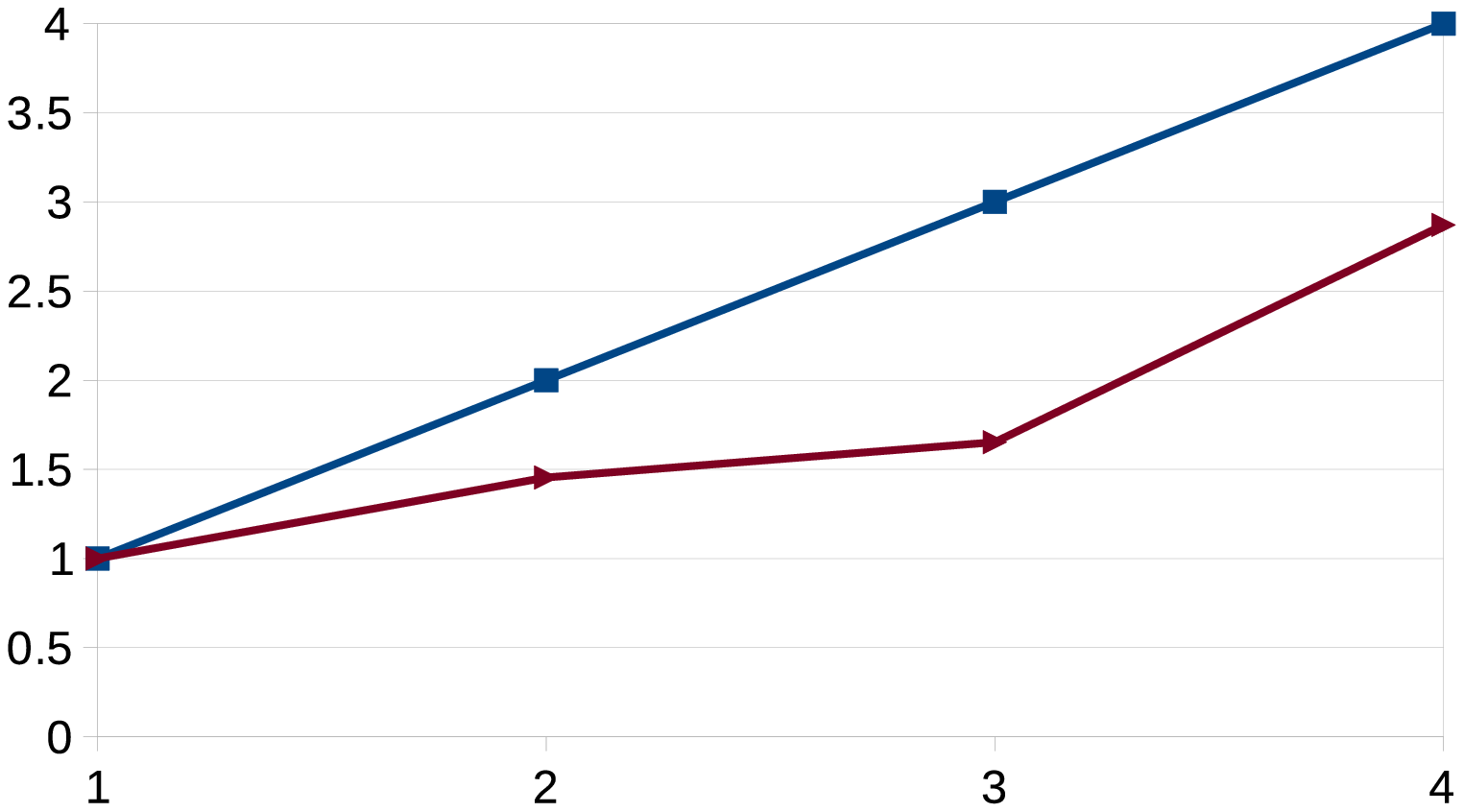} }\label{fig:speedup_yelp}}\hfill
    \subfloat[{\small CIFAR speedup}]{{ \includegraphics[trim={0 0.2cm 0.3cm 0.3cm},clip,width=0.32\textwidth ,height=2.6cm]{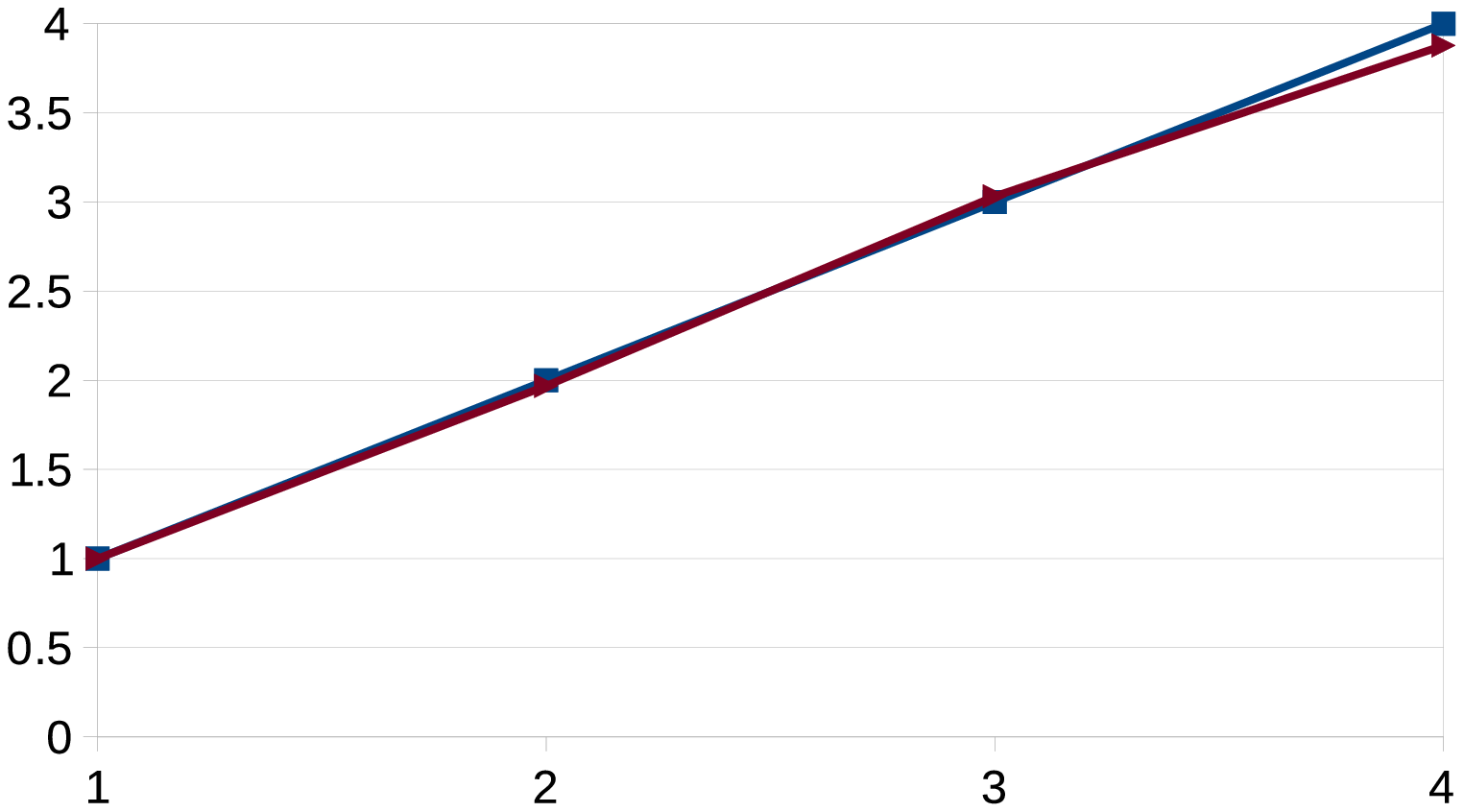} }\label{fig:speedup_cifar}}\hfill
    \subfloat[{\small ImageNet speedup}]{{ \includegraphics[trim={0 0.2cm 0.3cm 0.3cm},clip,width=0.32\textwidth, height=2.6cm]{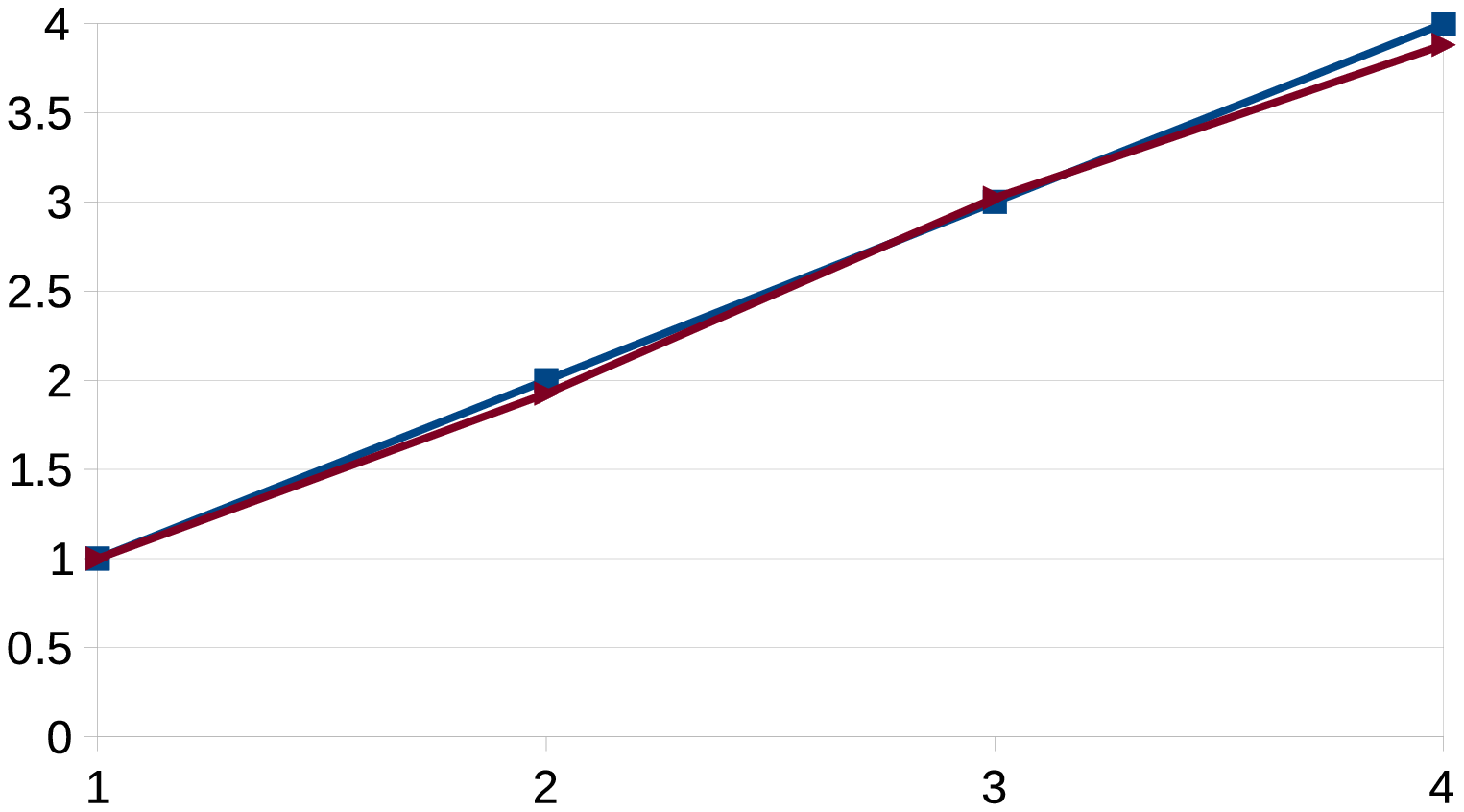} }\label{fig:speedup_imgnet}}\\[-1ex]
    \caption{\small Blue line is oracle linear speedup. Red line is the speedup performance of \Tool~.}
    \label{fig:speedup_gadei}
\end{figure*}
Table~\ref{tab:perf-baseline} records the single-learner performance baseline for comparison.
In Figure~\ref{fig:speedup_gadei}, the speedup performance of \Tool is plotted. When running on challenging commercial IBM Watson NLC workloads, \Tool can achieve on average 1.5X - 3X speedup when using up to 4 learners. When running on public image recognition benchmark tasks, \Tool achieves near linear speedup.
Dividing the total amount of data transferred between learners and parameter server by the total runtime, the memory bandwidth utilized by \Tool  is reported in Table~\ref{tab:bandwidth_utilized}. When running Watson workloads on 4 GPUs, \Tool sustains 36-55GB/s bandwidth, which is close to the hardware limit. 

\begin{table}[thbp]
\scriptsize
\centering
\begin{tabular}{|l|r|r|r|}
\hline
\multicolumn{1}{|l|}{}& \multicolumn{3}{c|}{\textbf{Memory Bandwidth Utilized(GB/s)}}  \\ \cline{2-4}
\textbf{Workload} & \textbf{$\lambda=2$} & \textbf{$\lambda=3$} & \textbf{$\lambda=4$} \\ \hline
\textbf{\jewel} & 18.56 & 27.25 & 36.48 \\ \hline
\textbf{\welltok} & 26.93 & 44.33 & 55.07 \\ \hline
\textbf{Age} & 32.94 & 36.76 & 46.42 \\ \hline
\textbf{Yelp} & 19.93 & 22.64 & 39.34 \\ \hline
\textbf{CIFAR} & 0.78 & 1.21 & 1.45 \\ \hline
\textbf{ImageNet} & 3.98 & 6.25 & 8.02 \\ \hline
\end{tabular}
\caption{\small Memory bandwidth utilization. $\lambda$ is the number of learners. }
\label{tab:bandwidth_utilized}
\end{table}

\vspace{-0.2cm}
\paragraph{Summary}
\Tool achieves linear speedup on public dataset and model, and achieves good speedup on challenging commercial workload. \Tool comes close to saturating the hardware memory bandwidth.

\subsection{Compare with DPT and mpiT}
\label{sec:eval-compare}
\begin{figure}
\small
\centering
\includegraphics[
keepaspectratio,scale=0.55]{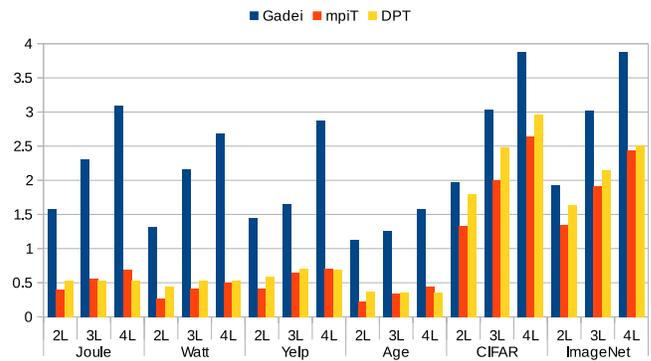} 
\caption{Speedup comparison among \Tool, mpiT and DPT. Y-axis is the speedup, X-axis represents different workload under $\lambda=2,3,4$ settings. \Tool consistently outperforms other tools.}
\label{fig:speedup_compare}
\end{figure}
We compare \Tool's speedup with that of mpiT and DPT, two state-of-the-art scale-up deep learning training frameworks. Figure~\ref{fig:speedup_compare} shows that \Tool consistently outperforms other tools. For NLC workload, DPT and mpiT actually slow down the execution. DPT has inferior performance because it is a SSGD style implementation, PS blocks all learners when updating weights, whereas \Tool implements ASGD style synchronization protocol. mpiT has inferior performance because it does not minimize memory copies (i.e., there is at least one extra copy from PS/learners to MPI runtime), and its message-passing style implementation makes HogWild! lock-free weights update impossible. 
\begin{table}[t]
\scriptsize
\centering
\begin{tabular}{|l|r|r|r|}
\hline
 \textbf{Workload}& \multicolumn{1}{c|}{\textbf{Gadei}} & \multicolumn{1}{c|}{\textbf{mpiT}} & \multicolumn{1}{c|}{\textbf{DPT}} \\ \hline
\textbf{\jewel} & 1 &16  & 32 \\ \hline\textbf{\welltok} & 1 &16  & 64 \\ \hline
\textbf{Age} & 1 & 32  & 32 \\ \hline
\textbf{Yelp} & 16 & 64 & 128 \\ \hline
\textbf{CIFAR} & 2 &  2& 8 \\ \hline
\textbf{ImageNet} & 2 & 8  & 8 \\ \hline
\end{tabular}
\caption{\small Minimum mini-batch size used by a system while still achieving speedup.}
\label{tab:stress_compare}
\end{table}
Table~\ref{tab:stress_compare} is the stress test: we decrease the batch size until there is no longer a speedup. Recall as described in Section~\ref{sec:smallbs_justification}, the smaller batch size a system can support, the higher probability a deep learning model can reach a desirable model accuracy. \Tool supports much smaller batch size, often the smallest possible size that is constrained only by the underlying model, than other tools.
\vspace{-0.2cm}
\paragraph{Summary}
\Tool significantly outperforms state-of-the-art open source implementations. It achieves good speedup on commercial workloads whereas existing open source implementations slow down the execution. On public benchmark, \Tool also outperforms existing open-source alternatives. \Tool can speedup workload using batch size that is an order of magnitude smaller than other state-of-the-art tools.

\subsection{Fault tolerance and over-subscription of learners}
\label{sec:eval-ft}
We randomly kill learners and verify \Tool can always finish training with desired model accuracy. In contrast, mpiT runs on top of MPI~\cite{mpi}, which traditionally does not provide fault-tolerance mechanism. DPT orchestrates multiple learners in the same process, thus when one learner fails, the entire process is killed. Further, nccl implements blocking collective operations, and their behavior in the presence of failure is not defined.\\
\Tool supports over-subscription of GPUs , i.e., run $\lambda$ learners over $N$ GPUs, where $\lambda > N $. Recurrent neural networks, such as long short-term memory (LSTM), are known to be difficult to efficiently parallelize  due to complicated computation dependencies. When such a model is adopted, GPUs usually operate at much lower efficiency compared to running CNNs. By oversubscribing GPUs, \Tool can fully utilize GPU resources. In another cognitive task, where LSTM is utilized, \Tool achieves 7-fold speedup when running 8 learners on 4 GPUs. We are also able to deploy 16 learners on 4 GPUs for NLC tasks to extrapolate their convergence behavior as if we had a 16-GPU server installation. 
\vspace{-0.2cm}
\paragraph{Summary}
\Tool supports fault-tolerance and GPU over-subscription. To the best of our knowledge, \Tool is the only scale-up deep learning system that supports fault-tolerance.

\section{Related work}
\label{sec:related}
Deep learning has seen tremendous success in image
recognition~\cite{krizhevsky2012imagenet}, natural language
processing~\cite{Feng2015}, speech translation~\cite{nmt:corr14} and
gaming~\cite{alphago}. The parameter server~\cite{parameterserver} based approach is the de facto method to scale
out such training tasks in a distributed environment. Several 
research works optimize PS performance and tackle fault-tolerance
problems~\cite{distbelief, adam, ps:osdi14} in a CPU-only enviroment. As GPU-based deep
learning frameworks~\cite{caffe, theano, torch} offer a much better
cost-effective solution, GPU-based scale-out PS architecture, such as
Mariana~\cite{mariana}, is optimized
for distributed GPU environment. GeePS~\cite{geeps:eurosys16}
overcomes the memory capacity limit on the GPU device by loading a part of
a model (instead of the entire model) that is necessary for
computation to GPU at a given point, and treats CPU memory as a large
data-cache. 

Different from previously published work, IBM Watson Cognitive Service
receives very different types of training data from worldwide customers and returns trained model individually. Consequently, it is imperative to set
conservative hyper-parameter configuration. This requires extremely high communication bandwidth that
renders scale-out solutions infeasible. State-of-the-art scale-up solutions, such as  nccl-based SSGD
implementation DPT\cite{dpt} and MPI-based ASGD implementation
mpiT\cite{mpiT} incur relatively large communication overhead due to intermediate data copy and execution stall. Our solution solves these issues by
minimizing the data copy and making the whole system highly concurrent.



\section{Conclusion}
\label{sec:conclusion}
In this paper, we focus on the system design challenges for emerging training-as-a-service (TaaS) workloads. By analyzing the characteristics of representative industrial workloads, we identify that to satisfy diverse customer requirements, a TaaS system needs to choose conservative hyper-parameter setup (e.g., small mini-batch size). We provide both empirical evidence and theoretical justification for such a design choice. We then characterize the communication bandwidth requirement for TaaS workloads and conclude that none of the state-of-the-art solutions can satisfy this requirement. 

We present \Tool, a scale-up deep learning framework, that maximizes communication bandwidth utilization in a tightly-coupled system. \Tool enables efficient multi-learner training for arbitrary type of neural networks (e.g., CNN, RNN). We further verify the correctness of the \Tool's communication protocol. Our evaluation results demonstrate that \Tool significantly outperforms state-of-the-art scale-up solutions on industrial workloads and public workloads, usually by an order of magnitude. In addition, \Tool provides fault-tolerance, which is missing in other scale-up solutions.

\clearpage
\newpage

\bibliographystyle{acm}
\bibliography{main}

\end{document}